\documentclass{article}
\usepackage{latexsym}
\usepackage{epsf,algorithm2e}
\usepackage{algorithmic}
\usepackage{slashbox}
\usepackage{amsmath,bbold,amsthm,amssymb}
\usepackage{graphicx,subfigure,url}
\usepackage{psfig}
\usepackage{psfrag}

\long\def\ignore#1{}
\newtheorem{theorem}{Theorem}[section]
\newtheorem{claim}{Claim}[section]
\newtheorem{definition}{Definition}[section]
\newcommand{\argmax}{{\text{argmax}}}
\newcommand{\vek}[1]{{\bf {#1}}}
\newcommand{\F}{{F}}
\newcommand{\vn}{{\vek{n}}}
\newcommand{\vv}{{\vek{v}}}
\newcommand{\vd}{{\vek{d}}}
\newcommand{\vl}{{\vek{\gamma}}}
\newcommand{\vx}{{\vek{x}}}
\newcommand{\vz}{{\vek{z}}}

\newcommand{\vzh}{{\hat{\vz}}}
\newcommand{\vy}{{\vek{y}}}
\newcommand{\vyi}{{\vy_i}}
\newcommand{\vxi}{{\vx_i}}

\newcommand{\pr}{{\mbox{$p$}}}
\newcommand{\prp}{{\text{p'}}}
\newcommand{\base}{{\phi}}
\newcommand{\cp}{{\text{C}}}
\newcommand{\cpr}{\cp_{\pr}}

\newcommand{\np}{{\phi}}

\newcommand{\range}[1]{{{\cal R'}_{#1}}}
\newcommand{\rangeC}[1]{{{\cal R}_{#1}}}
\newcommand{\dom}[1]{{{\cal D}_{#1}}}

\newcommand{\ra}{{\rightarrow}}

\newcommand{\Maj}{{\text{Maj}}}
\newcommand{\Potts}{{\text{Potts}}}
\newcommand{\eval}{{\text{V}}}

\newcommand{\pary}{{\text{NextLabel}}}
\newcommand{\parent}{{\text{parent}}}
\newcommand{\word}{{\text{TokenLabel}}}

\newcommand{\beforetoken}{{\text{BeforeToken}}}
\newcommand{\rooty}{{\text{End}}}
\newcommand{\domain}{{\text{domain}}}
\newcommand{\app}{{\text{domain adaptation}}}
\newcommand{\appcaps}{{\text{Domain Adaptation}}}
\newcommand{\mpc}{{\text{CI}}}
\newcommand{\indicate}[2]{[\![{#1}={#2}]\!]}  

\newcommand{\Maxlabel}{{\sc max}} 
\newcommand{\Additive}{{\sc sum}}  
\newcommand{\Pairwise}{{\sc paired}}  
\newcommand{\Majority}{{\sc majority}}  

\newcommand{\maxlabel}{{\sc max}} 
\newcommand{\additive}{{\sc sum}}  
\newcommand{\majority}{{\sc majority}}  
\newcommand{\alphapass}{{$\alpha$-pass}}
\newcommand{\alphaexpand}{{$\alpha$-expansion}}

\newcommand{\vapp}{{\hat{\vv}}}

\newcommand{\va}[2]{{\hat{\vv}^{#1#2}}}
\newcommand{\vya}{{\tilde{\vy}}}
\newcommand{\vva}{{\tilde{\vv}}}

\newcommand{\vmap}{{\vv^*}}
\newcommand{\vnp}{{\psi}}
\newcommand{\vcp}{{\text{C}}}

\newcommand{\vs}{{\vek{s}}}

\newcommand{\synpair}{{\sc potts}}
\newcommand{\synentropy}{{\sc entropy}}
\newcommand{\synms}{{\sc makespan}}
\newcommand{\synmajd}{{\sc maj-dense}}
\newcommand{\synmajs}{{\sc maj-sparse}}
\newcommand{\synmsq}{{\sc makespan2}}

\title{Generalized Collective Inference with Symmetric Clique Potentials}
\author{Rahul Gupta\hspace*{0.5in}Sunita Sarawagi\hspace*{0.5in}Ajit A.~Diwan\\\texttt{\{grahul,sunita,aad\}@cse.iitb.ac.in}\\Indian Institute of Technology, Bombay, India}
\date{}

\begin{document} 
\bibliographystyle{plain}
 
\maketitle 
\begin{abstract}

Many tasks like image segmentation, web page classification, and information
extraction can be cast as joint inference tasks in collective graphical models.
Such models exploit any inter-instance associative dependence to output more
accurate labelings. However existing collective models support very limited
kind of associativity --- like associative labeling of different occurrences of
the same word in a text corpus. This restricts accuracy gains from using such models.

In this work we make two major contributions. First, we propose a more general
collective inference framework that encourages various data instances
to agree on a set of {\em properties} of their labelings. Agreement is
encouraged through symmetric clique potential functions. We show that known collective models 
are specific instantiations of our framework with certain very simple
properties. We demonstrate that using non-trivial properties can lead to bigger gains,
and present a systematic inference procedure in our framework for a 
large class of such properties. In our inference procedure, we perform message passing on the cluster graph,
where property-aware messages are computed with cluster specific algorithms. 
Ordinary property-oblivious message passing schemes are intractable in such setups. 
We show that property conformance, as encouraged in our framework, provides an inference-only solution for \ignore{to the problem of} \app. 
Our experiments on bibliographic information extraction illustrate significant test error reduction over unseen \domain s.

Our second major contribution is a suite of algorithms to compute messages from
clique clusters to other clusters for a variety of symmetric clique
potentials (the {\em clique inference} problem). Our algorithms are exact for arbitrary cardinality-based clique potentials
on binary labels and for max-like and majority-like clique potentials on multiple labels.
For majority-like potentials, we also provide an efficient Lagrangian Relaxation based 
algorithm that compares favorably with the exact algorithm.
Moving towards more complex potentials, we show that clique inference becomes 
NP-hard for cliques with homogeneous Potts potentials. We present a $\frac{13}{15}$-approximation algorithm
with runtime sub-quadratic in the clique size. In contrast, the best known 
previous guarantee for graphs with Potts potentials is only $\frac{1}{2}$.
We perform empirical comparisons on real and synthetic data, and show that our
proposed methods for Potts potentials are an order of magnitude faster than the well-known
Tree-based re-parameterization (TRW) and graph-cut algorithms. We demonstrate that 
our Lagrangian Relaxation based algorithm for majority potentials beats the best applicable 
heuristic, ICM, in a variety of scenarios. 
\end{abstract}



\section{Introduction}
A variety of structured tasks such as image segmentation, information
extraction, part of speech tagging, text chunking, and named entity
recognition are modeled using Markov Random Fields (MRFs).  For
example, in information extraction, each sentence is treated as a MRF
that captures the dependency in the labels assigned to adjacent words
in the sentence. 

An example of such a setup is given in Figure~\ref{fig:ciExample} for the task
of named-entity extraction (NER). The base model in
Figure~\ref{fig:ciExampleBase} assigns a named-entity label such as Person,
Location, or Other  independently to each word in
the input. The structured model goes one step ahead and imposes a dependency
between labels of adjacent words, shown in
Figure~\ref{fig:ciExampleStructured} via chain-shaped MRF models. The model,
however, ignores long range and inter-sentence dependencies. The collective
model of Figure~\ref{fig:ciExampleCI} encourages the labels of different
occurrences of the same word to be the same. This is captured by connecting those
occurrences with blue cliques that encode associative dependencies. Variants of
these collective models have been proposed in the past few years for a variety
of information extraction tasks ~\cite{sutton04skip,bunescu04,krishnan06:effective,finkel05:Incorporating,gupta07}.
We look at other applications of collective graphical models in
Section~\ref{sec:app}.

\ignore{
In terms of tractability, the structured chain model
(e.g.~Figure~\ref{fig:ciExampleStructured}) can compute a chain's highest
scoring labeling (also called the MAP labeling) in $O(nm^2)$ where $n$ is the chain's length, and $m$ is the
number of labels. The collective model, however, quickly becomes intractable
because inference in graphs with multiple cycles is NP-hard in general. Thus,
the power to exploit associativity comes at a computational cost. 
}
\begin{figure}
\begin{center}
\subfigure[Input
sentences]{\label{fig:ciExampleInput}\includegraphics[width=0.3\textwidth]{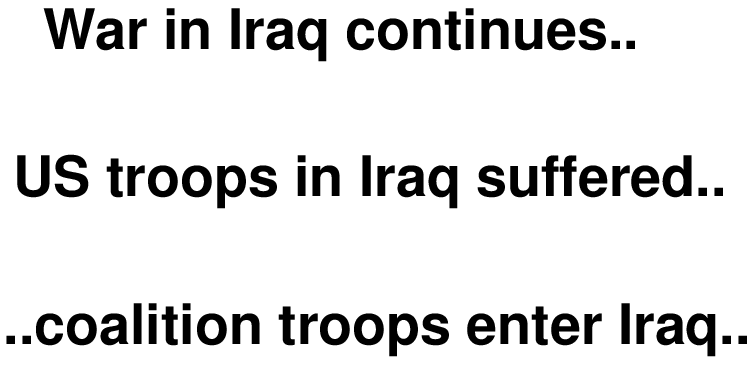}}
\subfigure[Base
Model]{\label{fig:ciExampleBase}\includegraphics[width=0.3\textwidth]{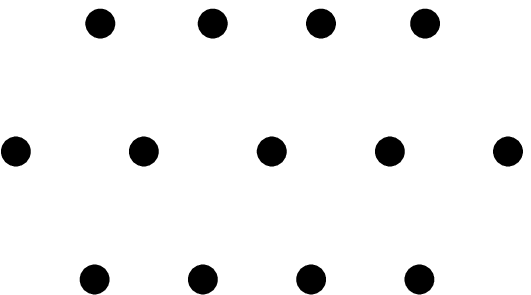}}
\subfigure[Structured
Model]{\label{fig:ciExampleStructured}\includegraphics[width=0.3\textwidth]{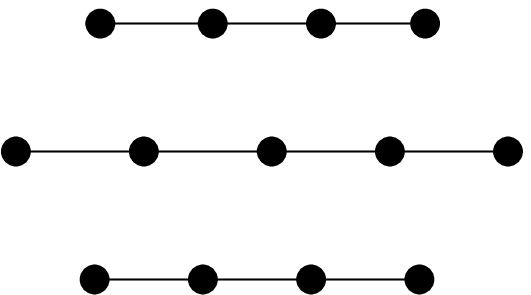}}
\subfigure[Collective
Model]{\label{fig:ciExampleCI}\includegraphics[width=0.3\textwidth]{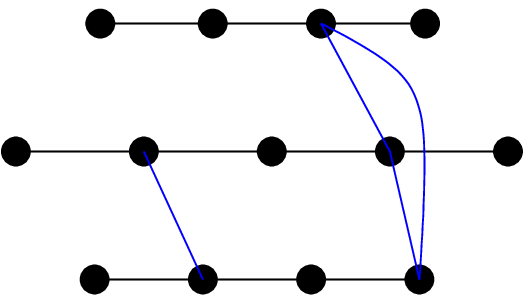}}
\subfigure[Cluster graph for the collective
model]{\label{fig:ciExampleCluster}\includegraphics[width=0.5\textwidth]{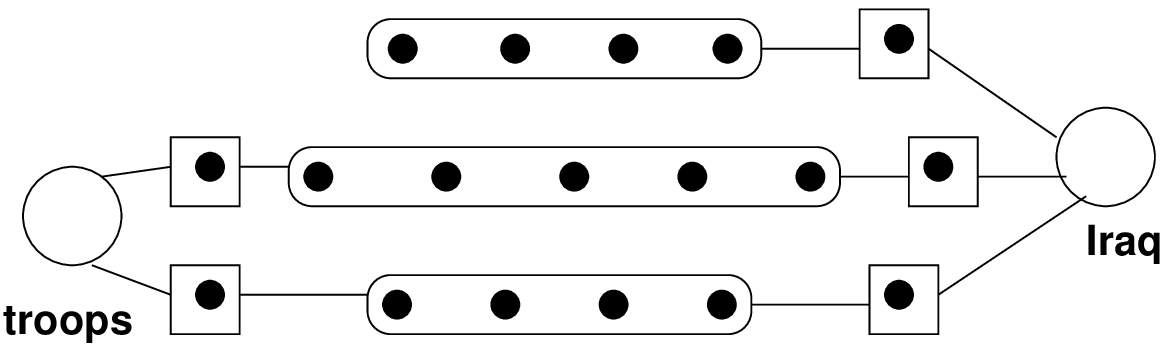}}
\caption{\label{fig:ciExample}Various models for named-entity recognition
illustrated on a small corpus. In Figure~\ref{fig:ciExampleCluster} the boxes
denote separators that link the clusters. The 'size' of a separator is the
number of nodes inside it.}
\end{center}
\end{figure}

\renewcommand{\arraystretch}{1.5}

\begin{table}[h] 
\begin{center} 
\begin{tabular*}{0.99\textwidth}{c|p{0.32\textwidth}|p{0.27\textwidth}|p{0.17\textwidth}} \hline 
{\bf Model} & {\bf Scoring function} & {\bf Inference algorithm} & {\bf Complexity} \\
\hline
Base & $\sum_{\text{chain }h}\sum_{u\in h}\phi_u(\vx, y_u)$  & $\argmax_{y_u}
\phi_u(\vx,y_u)$ for each vertex $u$ independently (Exact inference) & $|Y|$ per vertex \\
Structured & Base + $\sum_{h}\sum_{(u,v)\in h}\phi_{uv}(\vx,y_u,y_v)$  & Max-product separately on each chain (Exact)& $O(|V||Y|^2)$ per chain\\
Collective & Structured + $\sum_{\text{clique }c} \phi_c(y_c)$ & Message passing on the cluster graph (Approximate) & $O(|Y|^{\text{biggest separator}})$ per cluster \\
\hline
\end{tabular*}
\end{center} 
\caption{ \label{tab:models}
A brief summary of various graphical models for information extraction.}
\end{table}

The key ingredient in a collective model is the set of potentials used to tie
the individual MRFs together. These potentials, which can be defined over cliques
of arbitrary size, encourage their vertices to have the same/similar label.
We consider a special kind of clique potentials --- symmetric potentials.
Symmetric clique potentials are invariant under any permutation of their
arguments. This restriction is meant to keep inference tractable with such
potentials. Collective inference uses symmetric potentials that encourage all
the vertices to take the same label. This can be enforced by choosing an
appropriate bias function in the potential. Various kinds of symmetric potentials have
been used in collective inference thus far --- e.g.~Potts
potentials~\cite{sutton04skip,boykov01Fast} and Majority
potentials~\cite{krishnan06:effective}. We look at various families of symmetric clique potentials
in Section~\ref{sec:symPot}.

\subsection*{Properties-based Collective Inference Framework}

Figure~\ref{fig:ciExample} illustrates a highly common collective model in 
literature. Such a model enforces a very special kind of associativity --- that
labels of repeated occurrences of a word should be the same.
Restricting ourselves to such collective models does not allow us to exploit the
full power of collective inference, especially when unexploited associativity of a more
complex nature exists in the data, e.g.~all occurrences of Person should be
preceded by titular tokens such as Mr.~or Mrs.
We broaden the notion of collective inference to encourage for richer
forms of associativity amongst the labelings of multiple MRFs.  This
more general framework has applications in domain adaptation. We illustrate this
via an example.

\vspace*{0.1in}

\noindent{\bf Example:} Consider extracting bibliographic information from an author's
publications homepage using a model trained on a different set of
authors.  Typically, within each home (a domain) we expect consistency
in the style of individual publication records.  For example, we
expect that labelings of individual bibliographic records
(approximately) use the same ordering of labels (say $Title\rightarrow
Author*\rightarrow Venue$), regardless of what that ordering is. 
Another property whose conformance might be desirable is --- "the HTML tag
containing the Title of the publication". Different bibliographic records on the same page will
most probably format the title using the same HTML tag. Thus, we can encode this
associativity by biasing the labelings to be conformant wrt this property.
For both these properties, we only demand that the labelings agree on 
the property value, without caring for what the value actually is (which varies from
domain to domain). This allows us to use the same property on different domains, 
with varying formatting and authoring styles.

Now assume that we have an array of such conformance-promoting properties, and a
sequential chain model trained on a set of labeled domains. 
We show that an effective way of adapting the trained model to a new
domain is by labeling the MRFs in the new domain collectively while encouraging
the individual labelings to agree on our set of properties. This is an inference-only approach, unlike many
existing solutions for \app\ which require expensive model
re-training~\cite{blei02,blitzer2006Domain,mann07}. 
As we will see in Section~\ref{sec-domain}, using this properties-based
framework provides significant gains on a bibliographic information extraction
task.

To summarize, our collective inference framework consists of two new components
attached to any collection of MRFs:
\begin{itemize}
\item Properties: defined over the labelings of individual MRFs,

\item Potentials: defined on the values of properties of all MRFs such
that the potential value favors skewness in the frequencies of
property values.
\end{itemize}

We describe our framework in detail in Section~\ref{sec:properties}. 

The collective inference task in our framework is to choose a labeling of the
individual MRFs so as to maximize the sum of the scores of each MRF
and the potentials of each property.  This inference task is more
complicated than independently labeling each MRF, which are typically
simple tractable models like sequences.

We address the computational challenge by defining special forms of
decomposable properties and symmetric potentials that allow efficient
inference without sacrificing usability.
We exploit this special structure to design efficient MAP inference
algorithms.  Instead of ordinary belief propagation on the joint
graphical model, we define a cluster graph with two kinds of clusters
--- corresponding to tractable MRFs, and cliques with symmetric
potential functions. Two important aspects of this algorithm are as
follows.
First, we use combinatorial methods to compute cluster-specific
messages. In Section~\ref{sec:cliqueInference} we present exact and approximate
clique inference algorithms for a variety of symmetric potential functions.
These algorithms are used to compute max-marginal messages from a clique to its
neighboring clusters.
Second, we exploit the form of our properties to define new
intermediate message variables, and provide exact and approximate
algorithms for computing these special message values in Section~\ref{sec:algo}.  In contrast, a
na\"{\i}ve application of graphical model inference could lead
to an entire MRF instance being a separator.

In Section~\ref{sec:expts} our experiments on real tasks show that this form of message
passing is faster and more accurate than existing
inference methods that do not exploit the form of the potentials.

Finally in Section~\ref{sec:future}, we discuss some future directions for
collective inference and outline some important problems in the area.

\subsection*{Contributions}
Our first key contribution is a framework that encourages associativity between
properties of labelings of isolated MRFs. We show that the framework support a
large class of {\em decomposable} properties. Our properties are functions of a
data-instance and its labeling, in contrast to the existing associative setups
which model only very specific properties of only the instances. We give an
approximate inference procedure based on message passing on the cluster graph,
for computing the MAP labeling in our framework. Our procedure maintains
tractability by computing property-aware messages and invoking special
combinatorial algorithms at the cliques.

The second key contribution is a family of algorithms for various kinds of
symmetric clique potential functions. We give an $O(mn\log n)$ MAP algorithm for
cliques with arbitrary symmetric potentials, where $m$ is the number of labels,
and $n$ is the clique size. This algorithm is exact for max-like potentials, and
is $\frac{13}{15}$-approximate for Potts potentials. We show that this algorithm
can be generalized to an $O(m^2n\log n)$ algorithm while improving the
approximation bound to $\frac{8}{9}$. For majority-like potentials, we present
an LP-based exact algorithm with polynomial but expensive runtime. We present an
alternative approximate algorithm based on Lagrangian relaxation that is two
orders of magnitude faster and and provides close to optimal quality solutions
in practice.

Finally, we show that our suite of algorithms can be plugged into the
properties-based framework to achieve a highly expressive way for capturing
associativity. We illustrate this on a bibliographic information task where we
use properties to deploy our collective framework over unseen bibliographic
domains, and achieve significant error reductions.

\subsection*{Outline}

In Section~\ref{sec:app}, we give some real-life scenarios where collective
inference can be used to exploit associativity amongst isolated MRF instances.
We model associativity using symmetric potentials. In
Section~\ref{sec:symPot}, we describe three families of symmetric potentials,
that subsume the \Potts\ and linear \majority\ potentials.
Section~\ref{sec:properties} discusses our properties-based collective inference
framework in formal detail. Our framework ensures tractability of inference as
long as the properties are {\em decomposable}, a notion that we cover in
Section~\ref{sec:prop}. In Section~\ref{sec:algo}, we discuss the cluster message
passing algorithm to compute the MAP in our framework.
Section~\ref{sec:msgInstanceClique} presents our approach for exactly computing
property-aware messages from an MRF instance to a clique, and
Section~\ref{sec:approxMsg} contains practical approximations of this exact
computation. In Section~\ref{sec:cliqueAlgo}, we show that computing the reverse message
-- from a clique to a MRF instance, is the same as the {\em clique inference
problem}. Then in Section~\ref{sec:cliqueInference}, we present 
algorithms for solving the clique inference problem under a variety of symmetric
clique potentials. The two key algorithms presented are the \alphapass\
algorithm and a Lagrangian-relaxation based algorithm for \majority\ potentials,
in Sections~\ref{alg-alphaPass} and~\ref{sec:algoLR} respectively.
Section~\ref{sec:expts} contains experimental results of three types -- (a)
Establishing that our properties based framework leads to significant gains in a
\app\ task. (b) Our clique inference algorithms are better than applicable
alternatives and (c) The cluster message passing framework is a better way of
doing inference. Finally, Section~\ref{sec:future} contains conclusions and a
discussion of future work.

\section{Applications of Collective Inference}
\label{sec:app}
We review a few practical applications of collective inference in real-life
tasks. Both the tasks benefit when we introduce associative dependencies between
 labelings of isolated instances.
\subsection{Information Extraction}
\label{sec:app-ie}
Recall Figure~\ref{fig:ciExampleCI} for the task of named-entity
extraction. Let the potential function for an edge between adjacent word positions
$j-1$ and $j$ in document $i$ be $\phi_{ij}(y, y')$ and for
non-adjacent positions that share a word $w$ be $f_w(y, y')$.  
The goal during inference is to find a labeling $\vy$ where $y_{ij}$
is the label of word $x_{ij}$ in position $j$ of doc $i$, so as to
maximize:
\begin{equation}
\label{eq1}
\sum_{i,j} \phi_{ij}(y_{ij}, y_{i(j-1)}) + \sum_w\sum_{\text x_{ij}=x_{i'j'}=w}f_w(y_{ij}, y_{i'j'})
\end{equation}
The above inference problem gets intractable very soon with the
addition of non-adjacent edges beyond the highly tractable
collection of chain models of classical IE. 
\ignore{
For example, on the popular CoNLL dataset, this 
leads to more than $200$ cliques of size $10$ and above. }
Consequently, all prior
work on collective extraction for IE relied on generic approximation
techniques including belief propagation~\cite{sutton04skip,bunescu04}
, Gibbs sampling~\cite{finkel05:Incorporating} or stacking~\cite{krishnan06:effective}.

We present a different view of the above inference problem using
cardinality-based clique potential functions $\cp_w()$ defined over
label subsets $\vy^w$ of positions where word $w$ occurs.
We rewrite the second term in Equation~\ref{eq1} as
\begin{eqnarray*}
\frac{1}{2}\sum_{w}(\sum_{y,y'}f_w(y,y')n_y(\vy^w)n_{y'}(\vy^w)-\sum_y n_y(\vy^w)f_w(y,y))\\
=\sum_{w}\cp_w(n_{1}(\vy^w),\ldots n_{m}(\vy^w))
\end{eqnarray*}
 where $n_y(\vy^w)$ is the number of times $w$ is labeled $y$ in all
its occurrences.  The clique potential $\cp_w$ only depends on the
counts of how many nodes get assigned a particular label.  A useful
special case of the function is when $f_w(y,y')$ is positive only for
the case that $y=y'$, and zero otherwise. 

\ignore{
{\small
\begin{eqnarray*}
\cp_w(n_{1}(\vy^w),\ldots n_{m}(\vy^w)) &=& \sum_{y,y'}f_w(y,y')n_y(\vy^w) n_{y'}(\vy^w)\\
&=& \sum_{x_{ij}=x_{i'j'}=w}f_w(y_{ij}, y_{i'j'}) 
\end{eqnarray*}
}
}

\ignore{
$\cp_w(n_{1}(\vy),\ldots n_{m}(\vy)) = \sum_{y,y'}f_w(y,y')n_y(\vy)
n_{y'}(\vy)$ 
} 
\ignore{
A natural cluster graph for this task will have clusters of size two
for each of the adjacent edges, and clusters of size $n_w$ for each
word $w$ that appears $n_w\ge 2$ times.  The word clusters are
connected to exactly one adjacent edge cluster through singleton
separators.
Clearly, some of the word cliques can be very large, so
generic techniques for computing max-marginals will be impractical.
The sub-quadratic algorithms we propose in Section~\ref{sec-alg},
can aid in efficient belief propagation in the top-level cluster graph.
}
\ignore{
while retaining the entire
clique-level dependency in each message propagation step.
}

\ignore{
In addition to allowing for fast inference, this view also allows us
to include a more flexible and general set of potentials than pairwise
potentials.  Recently \cite{krishnan06:effective} have proposed 

, the kind of majority-based inference used in
 fit in this framework.
}


\subsection{Hypertext classification}
\label{sec-text}
In hypertext classification, the goal is to classify a document based
on features derived from its content and labels of documents it points
to.  A common technique in statistical relational learning to capture
the dependency between a node and the variable number of neighbors it
might be related to, is to define fixed length feature vectors out of
the neighbor's labels.  In text classification, most previous
approaches
~\cite{taskar02Discriminative,lu03Link,chakrabarti98Enhanced} have
created features based on the counts of labels in its
neighborhood. Accordingly, we can define the following set of
potentials: a node-level potential $\phi_i(y)$ that depends on
the content of the document $i$, and a neighborhood potential
$f(y, n_1(\vy^{O_i}),\ldots, n_m(\vy^{O_i}))$ that captures the
dependency of the label of $i$ on the counts in the label vector
$\vy^{O_i}$ of its out-links.
\begin{eqnarray*}
&&\sum_i (\phi_{iy_i}+ f(y_i,n_1(\vy^{O_i}),\ldots, n_m(\vy^{O_i})))\\
&=&\sum_i (\phi_{iy_i}+\sum_{y}\cp_{y}(n_1(\vy^{O_i}),\ldots, n_m(\vy^{O_i}))\indicate{y}{y_i})
\end{eqnarray*}

\cite{lu03Link} include several examples of such clique potentials,
viz.~the Majority potential $\cp_{y}(n_1,\ldots
n_m)=\phi(y,y_{max})$ where $y_{max} = \argmax_{y} n_y$, \ignore{
    Binary potential $\cp_{y}(n_1,\ldots n_m)=\sum_{y':n_{y'} > 0}\phi(y',y)$}
and the Count potential $\cp_{y}(n_1,\ldots n_m)=\sum_{y':n_{y'} >
0}\phi(y',y)n_{y'}$.  
Some of these potentials, for example, the Majority potential are not
decomposable as sum of potentials over the edges of the clique.  This
implies that methods such as 
\ignore{
tree re-parameterization
(TRW)~\cite{kolmogorov04Convergent} and graph-cuts \cite{boykov01Fast}}
TRW and graph-cuts
are not applicable.
\cite{lu03Link} rely on the Iterated Conditional Modes (ICM) method
that greedily selects the best label of each document in turn based on
the label counts of its neighbors.  

\section{Symmetric Clique Potentials}
\label{sec:symPot}

As seen in the example scenarios, our associative clique potentials depend only
on the number of clique vertices taking a value $v$, denoted by $n_v$, and not
on the identity of those vertices. In other words, these potentials are
invariant under any permutation of their arguments and derive their value from
the histogram of counts $\{n_v|\forall v\}$. We denote this histogram by the
vector $\vn$. Since the potentials only depend on the value counts, we also
refer to them as cardinality-based clique potentials in this paper. If a
cardinality-based clique potential is associative, then it is maximized when
$n_v=n$ for some $v$, i.e.~one value is given to all the clique vertices.

We have deliberately left the notion of a `value' vague at this point. For
existing collective models, e.g.~those mentioned in Section~\ref{sec:app}, a
value corresponds to a label. As we shall see, in our more general framework,
a value refers to a particular member in the range of a property function. For now, we can
assume wlog that a {\em value} is a member of some discrete finite set $V$.

We consider specific families of clique potentials, many of which are 
currently used in real-life tasks. In Section~\ref{sec:cliqueInference} we will
look at various potential-specific exact and approximate clique inference algorithms that
exploit the specific structure of the potential.

In particular, we consider the three types of clique potentials listed in
Table~\ref{tab:potentials}.

\begin{table}
\begin{center}
\begin{tabular}{c|c|c}
\hline
{\bf Name} & {\bf Form} & {\bf Remarks} \\
\hline
\Maxlabel & $\max_v f_v(n_v)$ & $f_v$ is a non-decreasing function \\
\Additive & $ \sum_v f_v(n_v)$ & $f_v$ non-decreasing. Includes the Potts potential = $\lambda\sum_v
n_v^2$ \\
\Majority & $f_a(\vn)$, where $a=\argmax_v n_v$ & $f_a$ is typically linear \\
\hline
\end{tabular}
\end{center}
\caption{Various kinds of symmetric clique potentials considered in this paper.
$\vn = (n_1,\ldots,n_{|V|})$ denotes the counts of various values among the clique
vertices.
\label{tab:potentials}
}
\end{table}

\ignore{
\begin{enumerate}
\item \Maxlabel: $\cp(n_1, \ldots, n_m) = \max_y f_y(n_y)$.  
\ignore{We saw an
example of such potentials for hypertext classification in
Section~\ref{sec-text}.}
These have been used 
\ignore{for information
extraction in~\cite{krishnan06:effective} and }
for analyzing associative networks in \cite{taskar04learning}.
\item \Additive: $\cp(n_1, \ldots, n_m) = \sum_y f_y(n_y)$.  A common
example of this function is $\lambda\sum_y n_y^2$ that arises out of cliques
with all edges following the same Potts potential with parameter $\lambda$. 
Another important
example of this class is negative entropy where $f_y(n_y)=n_y\log
n_y$.
\item \Majority: $\cp(n_1, \ldots, n_m) = f_a(\vn)$, where $a=\argmax_y n_y$. One 
common example of such a potential, as used in~\cite{krishnan06:effective} and ~\cite{lu03Link}, is 
$f_a(\vn) = \sum_{y}w_{ay}n_{y}$, where $\{w_{yy'}\}$ is an arbitrary matrix.
\ignore{
\item \Pairwise: $\cp(n_1, \ldots, n_m) = \sum_y \sum_{y'}
f_{yy'}(n_y, n_{y'})$ An example of this class is the quadratic entropy
function where $f_{yy'}(n_y, n_{y'})=s_{y,y'}n_yn_{y'}$ that would arise
out of summing any arbitrary potential matrix on the homogeneous edges of the clique. }
\end{enumerate}
}

\subsection{\Maxlabel\ clique potentials}

These clique potentials are of the form:
\begin{equation}
\cp(n_1,\ldots,n_{|V|})=\max_v f_v(n_v)
\end{equation}
 for arbitrary non-decreasing functions $f_v$. When $f_v(n_v)\triangleq
n_v$, we get the {\em makespan} clique potential which has roots in
the job-scheduling literature.  

In Section~\ref{alg-alphaPass}, we present an algorithm, called \alphapass, that
solves the clique inference problem for \maxlabel\ potentials exactly. The
algorithm runs in time $O(|V|n\log n)$, where $n$ is the clique size, 
Although \maxlabel\ potentials are not used directly in
real-life tasks, they are relatively easier potentials to tackle and provide key
insights to deal with the more complex \additive\ potentials. As we will see,
the \alphapass\ algorithm that we derive for this potential can be easily ported
to other more complex potentials. For the case of Potts potentials, we will
prove that the \alphapass\ algorithm provides a $\frac{13}{15}$-approximation.

\subsection{\Additive\ clique potentials}
\Additive\ clique potentials are of the form:
\begin{equation}
\cp(n_1,\ldots,n_{|V|}) = \sum_v f_v(n_v)
\end{equation}
These form of potentials includes the special case when the well-known
Potts model is applied homogeneously on all edges of a clique. Let
$\lambda$ be the Potts potential of assigning two nodes of an edge the
same value.  The summation of these potentials over a clique is
equivalent (up to a constant) to the clique potential:
\begin{equation}
\cp^\Potts=\cp(n_1,\ldots,n_{|V|})=\lambda\sum_v n_v^2
\end{equation}

The Potts model with negative $\lambda$ corresponds to the dis-associative case when
edges prefer the two end points to take different values.
\ignore{
 With negative $\lambda$, our objective function $\F(\vy)$ becomes concave
and its maximum can be easily found using a relaxed quadratic program
followed by an optimal rounding step as suggested in
\cite{ravikumar06Quadratic}.  We therefore do not discuss this case further.}
The more interesting case is when $\lambda$ is positive. For this case, we will
borrow the \alphapass\ algorithm for \Maxlabel\ potentials and show that it
gives a $\frac{13}{15}$-approximation.

\subsection{\Majority\ potentials}
\Majority\ potentials have been used 
for a variety of tasks such as link-based classification of web-pages \cite{lu03Link} and named-entity 
extraction \cite{krishnan06:effective}. A majority potential over a clique $C$
is parameterized by a $|V|\times |V|$ matrix $W=\{w_{vv'}\}$. 
The role of $W$ is to capture the co-existence of different value pairs in the same clique. 

Co-existence allows us to downplay `strict associativity' viz.~giving all vertices of a
clique the same value. The justification for co-existence is as follows.
Consider the conventional collective inference model for named-entity
recognition (Figure~\ref{fig:ciExampleCI}) where a value corresponds
to a label. Suppose the word 'America' occurs in a corpus multiple times. Then
all occurrences of 'America'
will be joined with an associative clique. However, some occurrences of America
correspond to Location, while others might correspond to an Organization, say
Bank of America. Thus we require most but not all vertices in the America clique
to be labeled similarly. This motivates the need for a clique potential with
scope for co-existence.

Coming back to $W$, a highly positive $w_{vv'}$ would suggest that the values $v$ and $v'$ should be allowed to co-exist in 
a clique, when $v$ is the majority value in the clique. We allow $w_{vv'}$ to be
negative to model mutual-exclusion amongst value pairs. Our algorithms work for
unrestricted $W$, although in practice the training procedure that learns $W$
might add some constraints.

We know define \majority\ potentials as:
\begin{equation}
\cp(n_1,\ldots,n_{|V|})=f_a(\vn),\ a=\argmax_v\ n_v
\end{equation}
 We consider linear majority potentials where $f_a(\vn) = \sum_{v}w_{av}n_{v}$. The matrix 
$W=\{w_{vv'}\}$ need not be diagonally dominant or even symmetric. 
Unlike Potts potential, \majority\ potential cannot be represented using edge potentials.

\section{Generalized Collective Inference Framework}
\label{sec:properties}
We now discuss our framework for generalized collective inference. Recall that
we wish to encourage the labelings of various isolated MRFs to agree on a set of
properties. Our generalized collective inference framework consists of three parts:
\begin{enumerate}
\item A collection of structured instance-labeling pairs
$\{(\vx_i,\vy_i)\}_{i=1}^N$ where each $\vy_i$\ is probabilistically
modeled using a corresponding Markov Random Field (MRF). Let
$\base(\vx_i,\vy_i)$ be a scoring function for assigning labeling
$\vy_i$ to $\vx_i$ using the MRF.  The scoring function decomposes
over the parts $c$ of the MRF as 
$\base(\vxi,\vy) = \sum_c \base_c(\vxi,\vy_c)$.
\item A set $P$ of properties where each property $p\in P$ includes in
its domain a subset $\dom{\pr}$ of MRFs and maps each labeling $\vy$ of
an input $\vx\in \dom{\pr}$ to a discrete value from its range
$\range{\pr}$.  Each property decomposes over cliques of the MRF.  We
discuss decomposable properties in Section~\ref{sec:prop}
\item A clique potential 
$\cpr(\{\pr(\vxi,\vyi)\}_{\vx_i\in\dom{p}})$ for each property $p$.
This potential is a symmetric function of its input.  We elaborated on
various symmetric potential functions in Section~\ref{sec:symPot}. These
potentials encourage conformity of properties across labelings of multiple
MRFs. 
\end{enumerate}

The collective inference task is to label the $N$ instances so as to
maximize the sum of the individual MRF specific scores and the clique
potentials coupling many MRFs via the property functions.  This is
given by:
\begin{equation}
\label{eqn:map}
\max_{(\vy_1,\ldots,\vy_N)} \sum_{i=1}^N \base(\vxi,\vyi) + \sum_{\pr\in P}
\cpr(\{\pr(\vxi,\vyi)\}_{i\in\dom{p}})
\end{equation}
Even for symmetric $\cpr$ and binary labels, Equation \ref{eqn:map} is NP-hard to
optimize. One well-known hard case is the Ising model, where each $\cpr$ is a Potts potential.
Thus we look at an approximate approach based on message
passing that we elaborate in Section~\ref{sec:algo}.

\ignore{
The MRF-based scoring function
$\base(\vxi,\vy)$ is decomposable over smaller components $c$ of $\vy$, i.e.
We define a property $\pr$ as Property values across all instances in its domain are coupled using
symmetric potential functions as we describe later in this section. Coupling 
encourages intra-domain conformity of instance labelings. 

In Sections~\ref{sec:prop}, \ref{sec:pot} and \ref{sec:ci} respectively, we 
elaborate upon properties, symmetric potentials and the generalized collective
inference problem.
}

\subsection{Decomposable Properties}
\label{sec:prop}
\ignore{ As an example of a tractable property, consider the case when
  $\vx$ and $\vy$ are both trees, and $c$ varies over nodes of the
  tree. Let $\pr(\vx,\vy_c)$ be a Boolean property defined only when
  $c$ is a leaf. The property is true if $y_c=\alpha$ for some fixed
  label $\alpha$ and false otherwise. Let $\pr(\vx,\vy) =
  \wedge_{\text{leaf }c }\pr(\vx,\vy_c)$. Then any max-marginal of
  $\pr$ can be computed easily.  } 

A property maps a $(\vx,\vy)$ pair to a discrete value in its
range. Typically, since $\vy$ is exponentially large in the size of
$\vx$, we cannot solve Equation~\ref{eqn:map} tractably without
constructing the value of a property from smaller components of $\vy$.
We define {\em decomposable} properties as those which can be broken
over the parts $c$ of the MRF of labeling $\vy$, just like
$\base$. Such properties can be folded into the message computation
steps at each of the MRFs, as we shall see in Section~\ref{sec:algo}.
We now formally describe decomposable properties:
\begin{definition}
\label{def:prop}
A decomposable property $\pr(\vx,\vy)$ is composed out of component
level properties $\pr(\vx,\vy_c,c)$ defined over parts $c$ of $\vy$.
$\pr:(\vx,\vy_c,c)\mapsto \rangeC{p} \cup \{\perp\}$ where the special
symbol $\perp$ means that the property is not applicable to
$(\vx,\vy_c,c)$.  $\pr(\vx,\vy)$ is composed as:
\begin{equation} 
\label{eq:prop}
\pr(\vx,\vy)
\triangleq \left\{ \begin{array}{ll} 
\emptyset & \text{if}~~ \forall c:  \pr(\vx,\vy_c,c)=~\perp \\ 
v & \text{if}~~ \forall c: \pr(\vx,\vy_c,c) \in \{v,\perp\} \\ 
\perp & \text{otherwise.}
\end{array}
\right.
\end{equation}
\end{definition}
The first case occurs when the property does not fire over any of the parts.
The last case occurs when $\vy$ has more than one parts where the
property has a valid value but the values are different.
The new range $\range{\pr}$ now consists of $\rangeC{\pr}$ and the two special symbols $\perp$ and $\emptyset$.

\ignore{
\begin{enumerate}
\item It defines an evaluator function $\eval_\pr$. For $\vx\in\dom{\pr}$ and a labeling
$\vy, \eval_\pr(\vx,\vy)$ is a set of pairs of the form $(c,v)$, where $c$ is a
part and $v\in\range{\pr}$.
\item It defines a special value $\perp\not\in\range{\pr}$, which denotes an
undefined property value.
\item The property function is defined as
\begin{equation}
\pr(\vx,\vy) \triangleq \left\{ \begin{array}{ll}
\emptyset & \eval_\pr(\vx,\vy)\mbox{ is empty} \\
\perp & \exists (c_1,v_1),(c_2,v_2\neq v_1)\in \eval_\pr(\vx,\vy) \\
v & \mbox{for any }(c,v)\in \eval_\pr(\vx,\vy)
\end{array}
\right.
\end{equation}
\end{enumerate}
\end{definition}
The evaluator function $\eval_\pr$ outputs the parts in the labeling where the
property fires, along with the property value at each part. The final property value is
well-defined if the property values at all locations are consistent.
}

We show that even with decomposable properties we can express many
useful types of regularities in labeling multiple MRFs arising in
applications like domain adaptation.

\paragraph{Example 1} 
We start with an example from the simple collective inference task of
\cite{sutton04skip,bunescu04,finkel05:Incorporating} of favoring the
same label for repeated words.  Let $\vx$ be a sentence and $\vy$ be a
labeling of all the tokens of $\vx$.  Consider a property $\pr$,
called {\bf \word}, which returns the label of a fixed token
$t$. Then, $\dom{\pr}$ comprises of all $\vx$ which have the token
$t$, and $\rangeC{\pr}$ is the set of all labels. Thus, if
$\vx\in\dom{\pr}$, then 
\begin{equation}
\pr(\vx,y_c,c) \triangleq \left\{ \begin{array}{ll}
y_c & \vx_c = t \\
\perp & \mbox{otherwise}
\end{array}
\right.
\end{equation}
and, given $\vy$ and $\vx \in \dom{\pr}$,
\begin{equation}
\pr(\vx,\vy) \triangleq \left\{ \begin{array}{ll}
y & \mbox{all occurrences of $t$ in $\vx$ are labeled with label $y$} \\
\perp & \mbox{otherwise}
\end{array}
\right.
\end{equation}

\paragraph{Example 2} 
\ignore{
Now consider a more complex example, where $\vy$ is a directed rooted
tree with labels on edges. Let there be a property $\pr$, called {\bf
\pary}, which returns the label of the parent edge of an edge labeled
$\alpha$ ($\alpha$ is a fixed label). For the sake of simplicity,
assume that there is a special edge from the root node to a dummy
root, and this edge is marked with a special label '\rooty'. So,
$\rangeC{\pr}$ comprises of '\rooty'\ and the set of all labels.
$\pr(\vx,\vy,e)= y_{\parent(e)}$ when $y_e = \alpha$, it is $\perp$
otherwise.  $\pr(\vx,\vy)$ is $\emptyset$ when there are no edges in
$\vy$ marked $\alpha$.  $\pr(\vx,\vy)$ is $v$ when all $\alpha$
labeled edges have $v$ labeled parent edges and it is $\perp$ otherwise.
} 
Next consider a more complex example that allows us to express
regularity in the order of labels in a collection of bibliography
records.  Let $\vx$ be a publications record and $\vy$ its labeling. Define
property $\pr$, called {\bf \pary}, which returns the first non-Other
label in $\vy$ after a Title.  A special label `\rooty' marks the end of
$\vy$. 
 So $\rangeC{\pr}$ contains '\rooty'\ and all labels except Other.
Thus, 
\begin{equation}
\pr(\vx,y_c,c) \triangleq \left\{ \begin{array}{ll}
\beta & y_c=\mbox{Title} \wedge y_{c+i}=\beta \wedge (\forall j:0<j<i:
y_{c+j}=\mbox{Other}) \\
\rooty & y_c=\mbox{Title} \wedge c \mbox{ is the last clique in }\vy \\
\perp & y_c\neq \mbox{ Title}
\end{array}
\right.
\end{equation}
Therefore,
\begin{equation}
\pr(\vx,\vy) \triangleq \left\{ \begin{array}{ll}
\emptyset & \vy\mbox{ has no Title}\\
\beta & \beta\mbox{ is the first non-Other label following {\bf each} Title in
}\vy \\
\perp & \mbox{otherwise}
\end{array}
\right.
\end{equation}

\paragraph{Example 3} 
In both the above examples the range $\range{p}$ of the properties was
labels.  Consider a third property, called {\bf \beforetoken}  whose range is the
space of tokens. This property returns the identity of the token before a Title
in $\vy$. So,
\begin{equation}
\pr(\vx,y_c,c) \triangleq \left\{ \begin{array}{ll}
x_{c-1} & y_c = \mbox{Title} \wedge (c>0) \\
\mbox{`Start'} & y_c=\mbox{Title} \wedge (c=0) \\
\perp & y_c\neq\mbox{Title}
\end{array}
\right.
\end{equation}
Therefore,
\begin{equation}
\pr(\vx,\vy) \triangleq \left\{ \begin{array}{ll}
\emptyset & \mbox{No Title in }\vy \\
\mbox{`Start'} & \mbox{The only Title in }\vy\mbox{ is at the beginning of }\vy
\\
t & \mbox{All Titles in }\vy\mbox{ are preceded by token }t \\
\perp & \vy\mbox{ has two or more Titles with different preceding tokens}
\end{array}
\right.
\end{equation}

\vspace*{0.2in}

Some important families of symmetric clique potentials
have been described in \cite{gupta07}, and recalled in Section~\ref{sec:symPot}.
We use two most widely used of those families, called \Potts\
and \Majority, defined as:
\begin{eqnarray*}
\cp^{\Potts}(\{v_1,\ldots,v_n\}) &\triangleq& \lambda\sum_{v'} n_{v'}^2 \\
\cp^{\Maj}(\{v_1,\ldots,v_n\}) &\triangleq& \sum_{v'} w_{\hat{v}{v'}} n_{v'},\ \
\hat{v}=\argmax_v n_v
\end{eqnarray*}
where $n_v$ is frequency of $v$ in the multiset $\{v_1,\ldots,v_n\}$ and $\lambda$, $w$ are fixed parameters of the potentials.
\ignore{
Such symmetric functions arise naturally in many applications.  For
example, for the \pary\ property in Example 2 of
Section~\ref{sec:prop} the \Potts\ potential can be used to reward
labeling all Titles with the same follow-up label.  In Example 1, \Majority\
can be used to allow preferential co-occurrence of pairs of
labels, as first described in~\cite{krishnan06:effective}

In Section~\ref{sec:cliqueAlgo} we will see how we exploit the
symmetric form of the potential to design efficient MAP inference
algorithms.
}
\ignore{
An
interesting class of clique potentials --- symmetrical clique
potentials, arises in many practical inference tasks. Symmetrical
clique potentials are of practical interest because they permit
efficient exact/approximate inference algorithms.

{\bf{Note: }}Since $\pr(\vx,\vy)$ can be $\perp$, depending on the task, the
clique potential can choose to ignore undefined values and compute the potential
over well-defined values. Alternatively, it can treat $\perp$ like just another
value. In some tasks, encouraging a property to be undefined for a majority of
instances can also quality as conformity.

}

\section{MAP Estimation in the Generalized Collective Inference Framework}
\label{sec:algo}
The natural choice for approximating the NP-hard objective in
Equation~\ref{eqn:map} is ordinary pairwise belief propagation on the joint model.
This approach does not work due to many reasons. First, some symmetric potentials like the \Majority\ potential cannot 
be decomposed along the edges. Second, property-aware messages cannot be computed for arbitrary message passing schedules. 
Third, cluster-membership information of vertices, which is very vital, is not exploited at all. 

Other approaches like the stacking based approach of~\cite{krishnan06:effective} are specific to 
particular symmetric potentials and do not exploit the full set of messages to compute a more accurate MAP.

Hence we adopt message passing on the cluster graph of the model as our
approach, akin to the one proposed by~\cite{duchi06Using}.
We create a top-level cluster graph model where the clusters
correspond to the $N$ instances and  $|P|$ property cliques.  The
cluster node of each instance is internally another nested MRF.  The cluster for
a property $\pr$ is a clique whose vertices correspond to instances in $\dom{\pr}$.
Figure~\ref{fig:cluster} illustrates an example with two properties and three
data instances.

\begin{figure}
\begin{center}
\psfrag{p1}{{\bf \word}}
\psfrag{p2}{{\bf \pary}}
\includegraphics{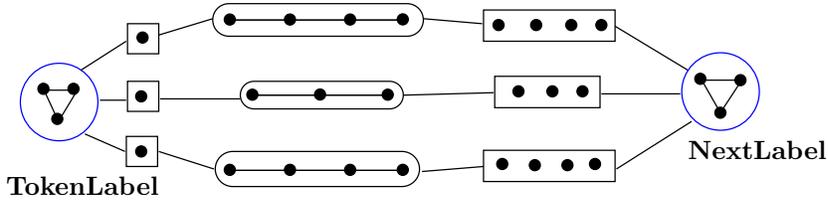}
\caption{\label{fig:cluster}Cluster graph for a toy example with three
chain-shaped MRF instances and two properties. The {\bf \word}\ property has thin
separators, while {\bf \pary}\ has separators that consist of the entire
instance. Both the properties have associative potentials defined on their
cliques (shown as blue circles).
}
\end{center}
\end{figure}

For complex properties like the {\bf \pary}\ property of Section~\ref{sec:prop}, 
the separator between a MRF cluster and a property cluster is the entire
instance. This is a major departure from known collective models such as the one
in Figure~\ref{fig:ciExampleCluster}. Known collective models use highly simple
properties, e.g.~{\bf \word} in Figure~\ref{fig:ciExampleCluster}, which lead
to single vertex separators because the property clique is incident on instances
only through a single token, which is known in advance. In the case of complex properties
like \pary, not only is the property clique incidence information missing, but the
clique's incidence is dependent on the property of the entire labeling and not
just a single token's label. This causes the entire instance to be a separator
between the property cluster and the instance MRF cluster. Therefore, na\"{\i}ve
message passing schemes whose runtime is exponential in the separator size are
inapplicable here. However we exploit the decomposability of properties to simplify message updates.

The setup of message passing on the cluster graph allows us to exploit potential-specific 
algorithms at the cliques, and at the same time work with any arbitrary clique
potential. It also allows intuitive computation of 
property-aware messages.

Let $m_{i\rightarrow p}$ and $m_{p\rightarrow i}$ denote message vectors from
instance $i$ to an incident property clique $\pr$ and vice-versa. Let
$v\in\range{\pr}$ denote a property value. Next we discuss how these messages are computed.

\ignore{
\begin{small}
\begin{eqnarray}
\label{eqn:msg1}
m_{i\rightarrow p}(v) = \max_{\vy:\text{\pr}(\vxi,\vy)=v}\base(\vxi,\vy) 
+ \sum_{\substack{p'\neq p\\\vxi\in\dom{p'}}}m_{p'\rightarrow i}(p'(\vxi,\vy))\ \\
\label{eqn:msg2}
m_{p\rightarrow i}(v) = \max_{\substack{(v_1,\ldots,v_n):\\v_i = v}}
\sum_{\substack{j\neq i\\\vx_j\in\dom{p}}}m_{j\rightarrow
p}(v_j) + \cpr(\{v_j\}_{\vx_j\in\dom{p}})
\end{eqnarray}
\end{small}
}

\subsection{Message from an Instance to a Clique}
\label{sec:msgInstanceClique}
The message $m_{i\rightarrow p}(v)$ is given by:
\begin{eqnarray}
\label{eqn:msg1}
m_{i\rightarrow p}(v) = \max_{\vy:\pr(\vxi,\vy)=v}\left(\base(\vxi,\vy) 
+ \sum_{\substack{p'\neq p:\\\vxi\in\dom{p'}}}m_{p'\rightarrow
i}(p'(\vxi,\vy))\right)
\end{eqnarray}
To compute $m_{i\rightarrow p}(v)$, we need to absorb the incoming
messages from other incident properties $p'\neq p$, and do the maximization. When a
property $p'$ is applicable to only a single fixed clique $c$ of the instance,
we can easily absorb the message $m_{p'\rightarrow i}$ by including it in the potential of
the clique, $\base_c$.  This is true, for instance, for the \word\ property in
the previous section when within a sentence the word does not repeat.
In the general case, absorbing messages that are applicable over
multiple cliques requires us to ensure that the cliques agree on the
property values  following Equation~\ref{eq:prop}. We will refer
to these as multi-clique properties.

We first present an exact extension of the message passing algorithm
within an instance MRF to enforce such agreement and later present
approximations.  Let $K$ be the set of multi-clique properties in
whose domain instance $\vx_i$ lies.  We are targeting applications
where $K$ is small.  

\subsubsection{Exact Messages}

\begin{figure}
\begin{center}
\psfrag{incoming2}{$m_{p_{|K|}\rightarrow i}$}
\psfrag{incoming1}{$m_{p_2\rightarrow i}$}
\psfrag{c1}{$c_1$}
\psfrag{cl}{$c_l$}
\psfrag{v1}{$(v_{11},\ldots,v_{1|K|})$}
\psfrag{vl}{$(v_{l1},\ldots,v_{l|K|})$}
\psfrag{c}{$c$}
\psfrag{cprime}{$c'$}
\psfrag{instancei}{$\mbox{Instance }i$}
\psfrag{s1}{$s_1$}
\psfrag{sl}{$s_l$}
\psfrag{sc}{$s$}
\psfrag{u}{$(u_1,\ldots,u_{|K|})$}
\psfrag{r}{$r$}
\psfrag{Mu}{$M(u_1,\ldots,u_{|K|})$}
\psfrag{finalmesg}{$m_{i\rightarrow p}$}
\psfrag{intmesg}{$I_{c\rightarrow c'}({\bf y}_s,u)$}
\includegraphics{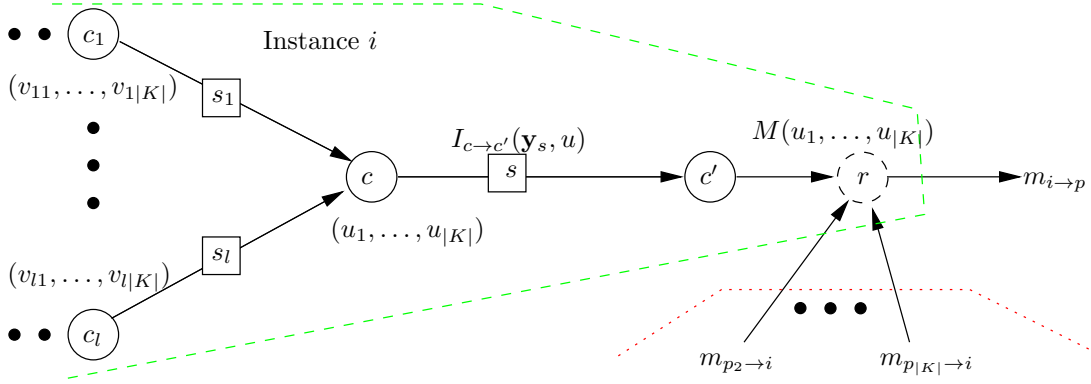}
\caption{\label{fig:internalMPA}Computation of the message $m_{i\rightarrow
p}$. Instance $i$ is incident to $|K|$ properties $p_1,\ldots,p_{|K|}$ where
$p=p_1$. The green portion shows the internal messages $I(.)$ in instance $i$.
Final message $m_{i\rightarrow p}$ is computed in terms of the aggregated
message $M(u_1,\ldots,u_{|K|})$ and any incoming messages $m_{p_j\rightarrow i}$,
$j>1$ (the red portion). }
\end{center}
\end{figure}

Figure~\ref{fig:internalMPA} shows the various messages involved in computing
the message $m_{i\rightarrow p}$. We describe the procedure step-by-step.
Consider an internal message $I_{c\ra c'}(\vy_s)$ between adjacent cliques $c$
and $c'$ {\em inside} the MRF of instance $i$ with $s$ as the separator. This is
computed in standard message passing as follows: 
\begin{equation}
\label{eq:mrf}
I_{c\ra c'}(\vy_s)=\max_{\vy_c \sim \vy_s} \base_c(\vy_c) + \sum_{j=1}^l I_{c_j\ra c}(\vy_{s_j})
\end{equation}
where $c_1\ldots c_l$ denote the $l$ neighboring cliques of clique $c$
excluding $c'$ and $s_1\ldots s_l$ are the corresponding separators.  To
handle multi-clique properties, we augment these internal messages to maintain state
about the set of properties already encountered in any partial
labeling up to $c$.  

For ease of explanation, first consider the case where
$|K|=1$, and let $p$ be the only property relevant for instance $i$.
 We maintain messages of the kind $I_{c\ra c'}(\vy_s, u)$, where $u \in
\range{p}$ is the called the {\em value argument} of $I$. 
These messages compute the following quantity: what is the score of the best
partial labeling $\vy_{part}$ up to $c$, which is consistent with $\vy_s$, such
that $p(\vx,\vy_{part})=u$ (as per Definition~\ref{def:prop}) if
we ignore the cliques beyond $c'$. 
To compute this message, we consider only the following entities: 
\begin{enumerate}
\item All local labelings $\vy_c$, consistent with $\vy_s$, such that $p(\vx,\vy_c,c)$ does not conflict
with $u$.
\item Incoming messages at $c$ (except from $c'$) of the kind $I_{c_i\rightarrow
c}(\vy_{s_i}, v_i)$ such that $v_i$ does not conflict with $u$, and all the $v_i$'s
together with $p(\vx,\vy_c,c)$ can produce a property value $u$ at $c$.
\end{enumerate}
Another way to look at it is as follows. Consider the green portion of Figure~\ref{fig:internalMPA},
and let $v_1,\ldots,v_l$ be $l$ candidate property values produced by some
partial labelings running up to $c_1,\ldots,c_l$ respectively. Also consider a
local labeling $\vy_c$ at $c$, consistent with $\vy_s$. Then computing the
message $I_{c\rightarrow c'}(\vy_s,u)$ is the same as {\em composing} $u$ by
picking a combination of $\vy_c$ and $v_1,\ldots,v_l$ such that
$p(\vx,\vy_c,c),v_1,\ldots,v_l$ can be amalgamated into $u$ via
Definition~\ref{def:prop}. If no such combination exists, then the message is
$-\infty$, else we return the combination with the highest total score.

Thus, depending on $u$, Equation~\ref{eq:mrf} is modified as follows:

\paragraph{Case 1: $u=\emptyset$: }
To get a value of $\emptyset$, the property should fire neither at $c$, nor up to
or at any of the $c_i$'s. That is, we need $p(\vx,\vy_c,c)=\emptyset$ and
incoming messages at $c$ whose value arguments are also $\emptyset$. 
Thus $I_{c\ra c'}(\vy_s,\emptyset)$ is computed as:
\begin{equation*}
\label{eq:mrfP}
   \max_{\substack{\vy_c \sim \vy_s\\p(\vx,\vy_c,c)=\perp}} \base_c(\vy_c) + \sum_{j=1}^l I_{c_j\ra c}(\vy_{s_j},\emptyset)
\end{equation*}

\paragraph{Case 2: $u \in \rangeC{p}$: }
In this case, $p(\vx,\vy_c,c)$ the value arguments of its incoming messages should be one of $u$ and $\emptyset$, with 
at least one of them being $u$ (otherwise all of them will be $\emptyset$ and we
will get $\emptyset$ at $c$). 
This will hold if the predicate $V_1(\vy_c,v_1,\ldots,v_l)$, defined below, is true.
\begin{equation}
V_1(\vy_c,v_1,\ldots,v_l) \triangleq (\forall j:v_j \in \{u,\emptyset\}) \wedge
(p(\vx,\vy_c,c)\in\{u,\perp\}) \wedge (p(\vx,\vy_c,c)=u~\vee~\exists v_j = u)
\end{equation}
Here $v_1,\ldots,v_l$ denote value arguments of the incoming messages at $c$ excluding the one from $c'$.
If the set $\{v_1,\ldots,v_l,p(\vx,\vy_c,c)\}$ contains two distinct values from
$\rangeC{p}$, then they will conflict and create $\perp$ at $c$ on composition.
Thus, $V_1$ precisely and completely represents the set of valid combinations
for producing $u$. Using $V_1$ we can compute $I_{c\ra c'}(\vy_s,u)$ as:
\begin{equation}
   \max_{\substack{\vy_c \sim \vy_s,v_1,\ldots,v_l:\\
   V_1(\vy_c,v_1,\ldots,v_l)}}\base_c(\vy_c) + \sum_{j=1}^l I_{c_j\ra c}(\vy_{s_j},v_j)
\end{equation}

\paragraph{Case 3: $u=\perp$: }
We can produce $\perp$ when either (a) value arguments of two or more incoming messages at $c$ conflict or (b) the property value at $c$ conflicts with 
the value argument of one of the incoming messages or (c) either the property
value at $c$ or any one of the value arguments is $\perp$.
The predicate $V_2(\vy_c,v_1,\ldots,v_l)$ returns true if any one of the above
outcomes hold:
\begin{eqnarray}
V_2(\vy_c,v_1,\ldots,v_l) &=& (\exists j: v_j=\perp)~\vee
(p(\vx,\vy_c,c)=\perp) \vee (\exists j,k: v_j\ne v_k \wedge v_j,v_k \in \rangeC{p}) \nonumber \\
&\vee &(p(\vx,\vy_c,c)=v_0,\ v_0\in \rangeC{p}) \wedge (\exists j: v_j \ne v_0,\ v_j \in \rangeC{p})
\end{eqnarray}
The outgoing message $I_{c\rightarrow c'}(\vy_s,u)$ is then:
\begin{equation}
I_{c\ra c'}(\vy_s,\perp)=   \max_{\substack{\vy_c \sim \vy_s,v_1,\ldots,v_l:\\ V_2(\vy_c,v_1,\ldots,v_l) }}
   \base_c(\vy_c) + \sum_{j=1}^l I_{c_j\ra c}(\vy_{s_j},v_j)
\end{equation}

After completing the internal message passing schedule, we can compute the
final aggregated message $M(u)$ by sending a message from the last clique
(wlog, say $c'$) to a dummy root clique $r$:
\begin{equation}
M(u) \triangleq I_{c'\rightarrow r}(\_,u)
\end{equation}
where the separator labeling is irrelevant because the message is to a dummy
clique $r$.

If $|K|=1$, then the message $m_{i\rightarrow p}(v)$ is simply $M(v)$.
This treatment generalizes nicely to the case when $|K|>1$.
We extend the internal message vector $I(.)$ for each combination of values of the
$|K|$ properties. Call it $I_{c\ra c'}(\vy_s,u_1,\ldots u_{|K|})$ where $u_j \in
\range{p_j}$. Let the final aggregated message at a dummy root
clique inside instance $i$ be $M(u_1,\ldots u_{|K|})$. 
 The outgoing message $m_{i\rightarrow p}(v)$ to property $p$ can now be computed
as:
\begin{equation}
m_{i\rightarrow p}(v)=\max_{\substack{u_1,\ldots u_{|K|}:\\ u_p=v}} M(u_1,\ldots
u_{|K|}) + \sum_{p'\neq p}m_{p'\rightarrow i}(u_{p'})
\end{equation}
The overhead for $|K|>1$ is that we have to absorb the incoming messages from
the other $|K|-1$ properties, shown in Figure~\ref{fig:internalMPA} in the red
portion.

\subsubsection{Approximations}
\label{sec:approxMsg}
Various approximations are possible to reduce the number of value combinations
for which messages have to be maintained.

Typically, the within-instance dependencies are stronger than the
dependencies across instances.  We can exploit this to reduce the
number of property values as follows: First, find the MAP of each
instance independently.  Only those property values that are fired in
the MAP labeling of at least one of the instances are considered in
the range.  In the application we study in Section~\ref{sec-domain},
this reduced the size of the range of properties drastically.

A second trick can be used when properties are associated with
labels that tend not to repeat in the MRF, e.g.~Title for the citation extraction task.  In that case, the value 
$\perp$ can be ignored. And, we can relax the consistency checks on
properties $p'$ that are being absorbed so that they can be absorbed
locally at each clique as follows.  First, normalize incoming messages
from $p'$ as $m_{p' \ra i}(v)-m_{p' \ra i}(\emptyset)$.  Next, absorb
the normalized message in the clique potential $\base_c(\vy_c)$ of all
cliques where $p'(\vx,\vy_c,c)=v$.  Finally, compute the outgoing
message by only keeping state over the values of the outgoing
property.  When the MAP does not contain any repeat firings of a
property, this method returns the exact answer.

A third option is to depend on generic search optimization tricks like
beam search and rank aggregation.  In beam search instead of
maintaining messages for each possible combination of property values,
we maintain only top-k most promising combinations in each message
passing step.

\ignore{
We do the
absorption by incorporating the incoming messages into the local base
potentials $\base_c$:
\begin{equation}
\base_c(\vxi,\vy_c) \leftarrow \base_c(\vxi,\vy_c) + m_{p'\rightarrow i}(p'(\vxi,\vy_c))
\end{equation}
where $c$ is a valid part for property $p'$. This step is possible because
$\prp$ is locally-computable and hence $\prp(\vxi,\vy_c) = \prp(\vxi,\vy)$. 

However, local absorption leads to a new problem: a message
$m_{p'\rightarrow i}(v')$ might get absorbed at multiple parts. To tackle this,
we do a message normalization step before absorption. Let $v'\in\range(\prp)$ be
the value that is most likely to be multiply absorbed. In our experiments we
will see that there is usually only one such value, and it can be predicted apriori. 
We then normalize the messages:
\begin{equation}
m_{p'\rightarrow i}(v'') \leftarrow m_{p'\rightarrow i}(v'') - m_{p'\rightarrow i}(v') 
\end{equation}
for all $v''\in\range(\prp)$. This zeroes out the message entry for $v'$ and
its multiple absorption has no effect.

After absorption the messages into $\base$, we need to perform the maximization
to compute the messages. This is done by simply computing the max-marginal 
$\max_{\vy:\pr(\vxi,\vy)=v} \base(\vxi,\vy)$ on the modified scoring function $\base$. 
Since $\pr$ is a tractable property, the max-marginal can be computed or
approximated efficiently.
}

\subsection{Message from a Clique to an Instance}
\label{sec:cliqueAlgo}
The message $m_{p\rightarrow i}(v)$ is computed as:
\begin{eqnarray}
\label{eqn:msg2}
m_{p\rightarrow i}(v) = \max_{\substack{(v_1,\ldots,v_n):\\v_i = v}}
\sum_{\substack{j\neq i\\\vx_j\in\dom{p}}}m_{j\rightarrow
p}(v_j) + \cpr(\{v_j\}_{\vx_j\in\dom{p}})
\end{eqnarray}
Message $m_{p\rightarrow i}(v)$ requires maximizing the objective in Equation
\ref{eqn:msg2}, which can be re-written as 
\begin{equation*}
-m_{i\rightarrow p}(v) +\max_{\substack{(v_1,\ldots,v_n)\\v_i = v}}
\sum_{j:\vx_j\in\dom{p}}m_{j\rightarrow
p}(v_j) + \cpr(\{v_j\}_{\vx_j\in\dom{p}})
\end{equation*}
The maximization subtask can be cast in terms of the general {\em clique
inference problem} defined as:
\begin{definition}
\label{def:cliqueInference}
Given a clique over $n$ vertices, with a symmetric clique potential
$\text{C}({v_1,\ldots,v_n})$, and vertex potentials $\psi_{jv_j}$ for all $j\leq
n$ and values $v_j$. Compute the value assignment with the highest potential:
\begin{equation}
\max_{v_1,\ldots,v_n} \sum_{j=1}^n \psi_{jv_j} + \text{C}({v_1,\ldots,v_n})
\end{equation}
\end{definition}
In our case,  $\psi_{jv}\triangleq m_{j\rightarrow p}(v)$ and
$\text{C}\triangleq\cpr$.
To compute $m_{p\ra i}(v)$, we can solve the clique inference problem with the
restriction $v_i=v$.
\ignore{
Even though $\cpr$ is symmetric, the clique inference problem is still NP-hard.
Two practically relevant symmetric potentials are the \Potts\ potential and the
\Majority potential, which we define next. 
If $v_1,\ldots,v_n$ is a value assignment to the clique
vertices, then let $n_{v'}$ denote the number of vertices assigned the value
$v'$ and let $\hat{v}$ be the value with the highest vertex membership. 
\Potts and Majority potentials are then defined as:
}

We are interested in the cases when the clique potential is \Potts\ or
\Majority, which were defined in Section~\ref{sec:properties}. These are most
popular potentials for real-life collective inference tasks.
\ignore{Note that \Potts\
potentials can be decomposed over the clique edges but \Majority\ cannot.}

In \cite{gupta07}, a $\frac{13}{15}$-approximate
clique inference algorithm, called $\alpha$-pass was presented for $\cp^{\Potts}$, along with an expensive 
polynomial-time exact algorithm for $\cp^{\Maj}$. 
\ignore{We first review $\alpha$-pass, as it} $\alpha$-pass can also be applied to arbitrary symmetric potentials and is
exact for binary valued properties. 
The time complexity of $\alpha$-pass is $O(|\range{p}|n\log n)$, as compared to 
$(|\range{p}|^2n^2)$ for ordinary belief propagation.

We next show that although $\alpha$-pass is also applicable for \Majority\
potentials, it lacks desirable theoretical guarantees. We then present a new approximate 
inference algorithm for $\cp^{\Maj}$ based on Lagrangian Relaxation which is much
faster than the exact algorithm yet produces almost-optimal scores in practice.

\ignore{
\subsubsection*{Notation}
Let $\psi_{jv} \triangleq m_{j\rightarrow p}(v)$ denote the vertex potential for vertex $j$ taking the value $v$.
}
\ignore{
\subsubsection{The $\alpha$-pass Algorithm}
In $\alpha$-pass, we compute a fast estimate of the best solution that has
exactly $k$ vertices with value $v$ ($k$ and $v$ are arbitrary). The $k$ best
vertices
to be marked with value $v$ are obtained in descending order of
$\psi_{jv} - \max_{v'\neq v}\psi_{jv'}$. The remaining $n-k$ nodes are assigned
their best non-$v$ values. Finally, we compute estimates for all $k$ and $v$,
and return the one with the best overall score.

The time complexity of $\alpha$-pass is $O(|\range{p}|n\log n)$, whereas 
ordinary belief propagation takes $(|\range{p}|^2n^2)$ time.
}

\section{Algorithms for Clique Inference}
\label{sec:cliqueInference}
In this section we explore various exact and approximate schemes for maximizing
the clique inference objective in Definition~\ref{def:cliqueInference} under 
a variety of symmetric potential functions. Of particular
interest are the \Potts\ and \Majority\ potentials, but some of the algorithms
are more general and apply to families of potentials.

These clique algorithms are called as subroutines while calculating the
messages from property cliques to instances, in accordance with
Equation~\ref{eqn:msg2}. Throughout this section, we assume that the clique
corresponds to a fixed property $\pr$ with range $\range{\pr}$. $R$ will be
short-hand for $|\range{\pr}|$.

We will use $\F(v_1,\ldots,v_n)$ to denote the clique inference objective. As
short-hand, we will denote $\F(v_1,\ldots,v_n)$ by $\F(\vv) = \vnp(\vv) +
\vcp(\vv)$ where $\vnp(\vv)$ is the vertex score of $\vv$ and the second term is the clique
score. Wlog assume that the vertex potentials are positive. Otherwise a constant
can be added to all of them and that will not affect the maximization. 
The best value assignment will be denoted by $\vmap$, and $\vapp$ will denote an
approximate solution.

\subsection{\alphapass\ Algorithm}
\label{alg-alphaPass}

We begin with \Maxlabel\ potentials. These potentials are not used in practice, but
clique inference for \Maxlabel\ potentials gives rise to the \alphapass\ algorithm
which has very interesting properties. Recall that a \Maxlabel\ potential is of
the form $\cp(\{v_1,\ldots,v_n\}) = \max_v f_v(n_v)$. The \alphapass\ algorithm
is described in Algorithm~\ref{fig:alphapass}. 

\begin{algorithm}
\KwIn{Vertex Potentials $\psi$, Clique Potential $\cp$, set $\range{\pr}$ of allowed values}
\KwOut{Value assignment $v_1,\ldots,v_n$}
Best = $-\infty$\;
$\vapp$ = nil\;
\ForEach{Value $\alpha\in\range{\pr}$}{
	Sort the vertices by the metric $\psi_{j\alpha} - \max_{v\in \range{\pr}, v\neq \alpha}\psi_{jv}$\;
	\ForEach{$k \in \{1,\ldots,n\}$} {
		Assign the first $k$ sorted vertices the value $\alpha$\;
		Assign the remaining vertices their individual best non-$\alpha$ value\;
		s $\leftarrow$ score of this assignment\;
		\If{s $>$ Best} {
			Best $\leftarrow$ s\;
			$\vapp$ $\leftarrow$ current assignment\;
		}
	}
}
\KwRet $\vapp$\;
\caption{The \alphapass\ algorithm}
\label{fig:alphapass}
\end{algorithm}

For each $(\alpha,k)$ combination, the \alphapass\ algorithm computes the best
$k$ vertices to get the value $\alpha$. Let $\va{\alpha}{k}$ denote the complete
assignment in the $(\alpha,k)^{th}$ step. Then it is easy to see that \alphapass\ runs in $O(|\range{\pr}|n\log n)$ time
by incrementally computing $\F(\va{\alpha}{k})$ from $\F(\va{\alpha}{(k-1)})$.
We now look at properties of \alphapass.

\begin{claim} 
\label{npclaim}
Assignment $\va{\alpha}{k}$ has the maximum vertex score over all $\vv$
where $k$ vertices are assigned $\alpha$, that is, $\vnp(\va{\alpha}{k}) =
max_{\vv:n_\alpha(\vv)=k}\vnp(\vv)$.
\proof This is easily seen by contradiction. If some other assignment $\vv\neq
\va{\alpha}{k}$ has the best vertex score, then it differs from $\va{\alpha}{k}$ in
the assignment of at least two vertices, one of which is assigned $\alpha$ in $\vv$ and non-$\alpha$ in
$\va{\alpha}{k}$. The converse holds for the other differing vertex. By swapping
their assignments, it is possible to increase the vertex score of $\vv$, a
contradiction.
\end{claim}
\begin{claim} 
\label{cpclaim}
For \maxlabel\ potentials, $\vcp(\va{\alpha}{k}) \ge f_\alpha(k)$.
\proof This is because the value $\alpha$ has a count of $k$ and the \maxlabel\
potential considers the maximum over all counts.
\end{claim}
\begin{theorem}
The \alphapass\ algorithm finds the MAP for \maxlabel\ clique
potentials.
\begin{proof}
Let $\vmap$ be the optimal assignment and let $\beta = \argmax_v
f_v(n_v(\vmap)),\ \ell = n_\beta(\vmap)$.  Let $\vapp$ be the assignment found by \alphapass.
We have:
\begin{eqnarray*}
\F(\vapp) &=& \max_{1\le \alpha \le |\range{\pr}|, 1 \le k \le
n}\F(\va{\alpha}{k}) \\
&\ge& \F(\va{\beta}{\ell})\\
&=&\vnp(\va{\beta}{\ell})+\vcp(\va{\beta}{\ell})\\
&\ge& \vnp(\va{\beta}{\ell})+f_\beta(\ell) \\
&=&\vnp(\va{\beta}{\ell})+\vcp(\vmap)\\
&\ge& \vnp(\vmap)+\vcp(\vmap) \\
&=& \F(\vmap)
\end{eqnarray*}
The second and third inequalities follows from Claims~\ref{cpclaim}
and~\ref{npclaim} respectively.
\end{proof}
\end{theorem}

Thus, \alphapass\ finds the optimal assignment for the \maxlabel\ family of
potentials in $O(Rn\log n)$ time. We now move on to \additive\
potentials.

\subsection{Clique Inference for \Additive\ Potentials}
We will mainly focus on the \Potts\ potential, which is arguably the most popular
member of the \additive\ family. \Potts\ potential is given by $\vcp(\vv) =
\lambda\sum_v n_v^2$.

When $\lambda < 0$, the clique edges prefer the two end points to take different
values. With
negative $\lambda$, our objective function $\F(\vv)$ becomes concave
and its maximum can be easily found using a relaxed quadratic program
followed by an optimal rounding step as suggested in
\cite{ravikumar06Quadratic}.  We therefore do not discuss this case further.
The more interesting case is when $\lambda$ is positive. We show that
finding $\vmap$ now becomes NP-hard.
\begin{theorem}
When $\vcp(\vv)=\lambda\sum_v n_v^2, \lambda > 0$, finding the MAP
assignment is NP-hard.
\end{theorem}
\begin{proof}
Let $R \triangleq |\range{\pr}|$.
We prove hardness by reducing from the NP-complete {\em exact cover by 3-sets}
problem~\cite{ps82:combinatorial} of deciding if exactly $\frac{n}{3}$
of $R$ subsets $S_1,\ldots,S_R$ of 3 elements each from $U=\{e_1,\ldots
e_n\}$ can cover $U$.  We let elements correspond to vertices and sets to
values.  Assign $\vnp_{iv}=2n\lambda$ if $e_i\in S_v$ and 0 otherwise.
MAP score will be $(2n^2+3^2\frac{n}{3})\lambda$ iff we can find an
exact cover.
\end{proof}
The above proof establishes that there cannot be an algorithm that is
polynomial in both $n$ and $R$.  But we have not ruled out
algorithms with complexity that is polynomial in $n$ but exponential in 
$R$, say of the form $O(2^Rn^c)$ for a constant $c$.

We next propose approximation schemes.  Unlike for general graphs
where the Potts model is approximable only within a factor of
$\frac{1}{2}$~\cite{boykov01Fast,klienberg02Approx}, we show that for
cliques the Potts model can be approximated to within a factor of
$\frac{13}{15}\approx 0.86$ using the \alphapass\ algorithm.
We first present an easy proof for a
weaker bound of $\frac{4}{5}$ and then move on to a more
detailed proof for the $\frac{13}{15}$ bound.  Recall that the optimal
assignment is $\vmap$ and the assignment output by \alphapass\ is $\vapp$.

\begin{theorem}
\label{th-alpha1}
 $\F(\vapp) \ge \frac{4}{5}\F(\vmap)$.
\end{theorem}
\begin{proof}
Without loss of generality assume that the counts in $\vmap$ are $n_1
\ge n_2 \ge \ldots \ge n_R$, where $R=|\range{\pr}|$. Then $\sum_v n_v^2 \le n_1\sum_v n_v =
nn_1$.
\begin{eqnarray*}
\F(\vapp) &\ge& \F(\va{1}{n_1}) =
\vnp(\va{1}{n_1})+\vcp(\va{1}{n_1}) \\ &\ge&
\vnp(\vmap)+\vcp(\va{1}{n_1}) ~~~\text{(from Claim~\ref{npclaim})}\\
&\ge& \vnp(\vmap)+\lambda n_1^2~~~\text{(since $\lambda > 0$)}\\
&\ge& \vnp(\vmap)+\vcp(\vmap)-\lambda n_1 n + \lambda n_1^2 \\
&\ge& \F(\vmap)-\lambda n^2/4
\end{eqnarray*}
Now consider the two cases where $\F(\vmap) \ge \frac{5}{4}\lambda
n^2$ and $\F(\vmap) < \frac{5}{4}\lambda n^2$. For the first case we
get from above that $F(\vapp) \ge \F(\vmap)-\lambda n^2/4 \ge
\frac{4}{5}\F(\vmap)$.  For the second case, we know that the score 
$\F(\va{m}{n})$ where we assign all vertices the last value is at least
$\lambda n^2$ and  thus $\F(\vapp) \ge \frac{4}{5}\F(\vmap)$.
\end{proof}

We now state the more involved proof for showing that \alphapass\ actually
provides a tighter approximation bound of $\frac{13}{15}$ for Potts potentials.

\begin{theorem}
\label{theorem:alphapass}
 $\F(\vapp) \ge \frac{13}{15}\F(\vmap)$.
\end{theorem}
\begin{proof}
The proof is by contradiction. Suppose there is an instance where $\F(\vapp)\le
\frac{13}{15}\F(\vmap)$.
Wlog assume that $\lambda=1$ and $n_1 \geq n_2 \geq \ldots \geq n_k > 0,\ (2\leq k
\leq R)$
be the non-zero counts in the optimal
solution and let $\vnp^*$ be its vertex potential. 
Thus $\F(\vmap) = \vnp^* + n_1^2 + n_2^2 + \ldots + n_k^2$.

Now, $\F(\vapp)$ is at least $\vnp^* + n_1^2$ (ref.~Claim \ref{npclaim}). This implies
$\frac{\vnp^* + n_1^2}{\vnp^* + n_1^2 + \ldots + n_k^2}\le \frac{13}{15}$, i.e.~
$2(\vnp^* +n_1^2) < 13(n_2^2 + \ldots + n_k^2)$  or
\begin{equation}
\label{eqn:proof_eqn1}
\vnp^* < \frac{13}{2}(n_2^2 + \ldots + n_k^2) - n_1^2
\end{equation}

Since $k$ values have non-zero counts, and the vertex score is $\vnp^*$, at least
$\vnp^*/k$ of the vertex score is assigned to one value. Considering a solution where
all vertices are assigned to this value, we get  $\F(\vapp) \geq \vnp^*/k + n^2$. 

Therefore $\F(\vmap) > 15/13(n^2 + \vnp^*/k)$.

Since $\F(\vmap) = \vnp^* + n_1^2 + \ldots + n_k^2$,
we get:
\begin{equation}
\label{eqn:proof_eqn2}
\vnp^* > \frac{15kn^2-13k(n_1^2 + \ldots + n_k^2)}{13k-15}
\end{equation}

We show that Equations \ref{eqn:proof_eqn1} and \ref{eqn:proof_eqn2} contradict each other.
It is sufficient to show that for all $n_1 \geq \ldots \geq n_k \geq 1$,
\begin{equation*}
\frac{15kn^2-13k(n_1^2 + \ldots n_k^2)}{13k-15} \geq \frac{13}{2}(n_2^2 + \ldots + n_k^2) - n_1^2
\end{equation*}
Simplifying, this is equivalent to
\begin{equation}
kn^2 - \frac{13}{2}(k-1)(n_2^2 +\ldots + n_k^2) - n_1^2 \geq 0.
\end{equation}

Consider a sequence $n_1, \ldots, n_k$ for which the expression on the left hand side is
minimized. If $n_i > n_{i+1}$ then we must have $n_{l} = 1\ \forall l\ge i+2$. 
Otherwise, replace $n_{i+1}$ by $n_{i+1}+1$ and decrement $n_j$ by 1,
where $j$ is the largest index for which $n_j > 1$. This gives a new sequence
for which the value of the expression is smaller. Therefore the sequence
    must be of the form $n_i = n_1$ for $1 \leq i < l$ and $n_i = 1$ for $i > l$, for
    some $l \geq 2$. Further, considering the expression as a function of $n_l$,
    it is quadratic with a negative second derivative. So the minimum occurs 
    at one of the extreme values $n_l = 1$ or $n_l = n_1$.
    Therefore we only need to consider sequences of the form $n_1,\ldots,n_1,1,\ldots,1$
    and show that the expression is non-negative for these.

    In such sequences, differentiating with respect to
    $n_1$, the derivative is positive for $n_1 \geq 1$, which means that the expression
    is minimized for the sequence $1,\ldots,1$. Now it is easy to verify that
    it is true for such sequences. The expression is zero only for the sequence
    $1,1,1$, which gives the worst case example.
\end{proof}

The next theorem that the analysis in Theorem~\ref{theorem:alphapass} is tight.
We present a pathological example where \alphapass\ gives a solution which is
exactly $\frac{13}{15}$ of the optimal.

\begin{theorem}
\label{tight-instance}
The approximation ratio of $\frac{13}{15}$ of the \alphapass\
algorithm is tight.
\end{theorem}
\begin{proof}
We show an instance where this is obtained.  Let $R=n+3$ and
$\lambda=1$. For the first $n/3$ vertices let $\vnp_{u1}=4n/3$, for the
next $n/3$ vertices let $\vnp_{u2}=4n/3$, and for the remaining $n/3$ let
$\vnp_{u3}=4n/3$.  Also for all vertices let $\vnp_{u(u+3)}=4n/3$. All
other vertex potentials are zero.  The optimal solution is to assign the
first three values $n/3$ vertices each, yielding a score of
$4n^2/3+3(\frac{n}{3})^2=5n^2/3$.  The first \alphapass\ on value 1,
where initially a vertex $u$ is assigned its vertex optimal value $u+3$,
will assign the first $n/3$ vertices 1.  This keeps the sum of total vertex
potential unchanged at $4n^2/3$, the clique potential increased to
$n^2/9+2n/3$ and total score = $4n^2/3+n^2/9+2n/3=13n^2/9+2n/3$. No
subsequent iterations with any other value can improve this
score.  Thus, the score of \alphapass\ is $\frac{13}{15}$ of the
optimal in the limit $n\rightarrow \infty$.
\end{proof}

\subsubsection{\alphaexpand}
In general graphs, a popular method that provides the approximation
guarantee of 1/2 for the Potts model is the graph-cuts based $\alpha$ expansion
algorithm~\cite{boykov01Fast}.  We explore the behavior of this
algorithm for \Potts\ potentials.

In this scheme, we start with any initial assignment --- for example, all
vertices are assigned the first value as suggested in ~\cite{boykov01Fast}. Next, for
each value $\alpha$ we perform an $\alpha$ expansion phase where we
switch the assignment of an optimal set of vertices to $\alpha$ from their
current value.  We repeat this until in a round over the $R$ values,
no vertices switch their assignment.

For graphs whose edge potentials form a metric, an optimal $\alpha$
expansion move is based on the use of the mincut
algorithm of \cite{boykov01Fast} which for the case of cliques can be
$O(n^3)$. 

We next show how to perform optimal $\alpha$ expansion
moves more efficiently for all kinds of \additive\ potentials.
\paragraph{An $\alpha$ expansion move}
Let $\vva$ be the assignment at the start of this move. For each value
$v \ne \alpha$ create a sorted list $S_v$ of vertices assigned $v$ in
$\vva$ in decreasing order of $\vnp_{i\alpha} - \vnp_{iv}$.  If in an
optimal move, we move $k_v$ vertices from $v$ to $\alpha$, then it is
clear that we need to pick the top $k_v$ vertices from $S_v$.  Let $r_i$
be the rank of a vertex $i$ in $S_v$.  Our remaining task is to decide
the optimal number $k_v$ to take from each $S_v$.  We find these using
dynamic programming.  Without log of generality assume $\alpha=R$ and
$1,\ldots, R-1$ are the $R-1$ values other than $\alpha$.

Let $D_j(k)$ denote the best score with $k$ vertices assigned values from
$1\ldots j$ switched
to $\alpha$.
We compute
\begin{eqnarray*}
D_j(k) = \max_{l \le k, l \le
n_{j}(\vva)} D_{j-1}(k-l) + f_{j}(n_j(\vva)-l) \\+ \sum_{i':r_{i'}\le
l}\vnp_{i'\alpha} + \sum_{i':r_{i'} > l}\vnp_{i'j}
\end{eqnarray*}
From here we can calculate the optimal number of vertices to switch to
$\alpha$ as $\argmax_{k \le n-n_\alpha(\vva)}D_{R-1}(k) +
f_\alpha(k+n_\alpha(\vva))$.
\begin{theorem}
The \alphaexpand\ algorithm provides no better approximation guarantee
than 1/2 even for the special case of homogeneous Potts potential on cliques.
\end{theorem}
\begin{proof}
Consider an instance where $R=k+1$, and $\lambda=1$. Let $\vnp_{u1} = 2n/k$ for all $u$ 
and for $k$ disjoint groups of $n/k$ vertices each, let
$\vnp_{u,i+1} = 2n$ for the vertices in the $i^{th}$ group. All other vertex potentials are zero.
Consider the solution where every vertex is assigned value $1$. This assignment is locally optimal wrt
any \alphaexpand\ move, and its score is $n^2(1+2/k)$.
However, the exact solution assigns every vertex group its value, with a score $n^2(2+1/k)$ , 
thus giving a ratio of $1/2$ in the limit.
\end{proof}

We next present a generalization of the \alphapass\ algorithm that
provides provably better guarantees while being faster than
\alphaexpand.

\subsubsection{Generalized \alphapass\ algorithm}
In \alphapass\, for each value $\alpha$, we go over each count $k$ and
find the best vertex score with $k$ vertices assigned value $\alpha$.  We
generalize this to go over all value combinations of size no more than
$q$, a parameter of the algorithm that is fixed based on the desired
approximation guarantee.  

For each value subset $A\subseteq \range{\pr}$ of size no more than $q$, and for each count
$k$, maximize vertex potentials with $k$ vertices assigned a value from set
$A$.  For this, sort vertices in decreasing order of $\max_{\alpha \in A}
\vnp_{i\alpha} - \max _{v \not\in A} \vnp_{vuy}$, assign the top $k$
vertices their best value in $A$ and the remaining their best value not
in $A$.  The best solution over all $A,k$ with $|A|\le q$ is the final
assignment $\vapp$.

The complexity of this algorithm is $O(nR^q\log n)$. 
\ignore{Since $R$ is
typically smaller than $n$, this is a useful option for large
cliques.} 
In practice, we can use heuristics to prune the number 
of value combinations. Further, we can make the following claims
about the quality of its output. \ignore{We skip the proof because it is quite involved.}
\begin{theorem}
\label{theorem:genalphapass}
 $\F(\vapp) \ge \frac{8}{9}\F(\vmap)$.
 \begin{proof}
 This bound is achieved if we run the algorithm with $q=2$. Let the optimal solution have 
 counts $n_1\geq n_2\geq\ldots\geq n_R$ and let its vertex potential be $\vnp^*$.
 For simplicity let $a = n_1/n$, $b = n_2/n$ and $c = \vnp^*/n^2$.
 Then $\F(\vmap)/n^2  \leq  c + a^2 + b(1-a)$, $\F(\vapp)/n^2 \geq c + a^2$  and $\F(\vapp)/n^2 \geq c + \frac{(a+b)^2}{2}$.

 \noindent{\bf{Case 1:}} $a^2 \geq \frac{(a+b)^2}{2}$.
 Then $\F(\vmap)-\F(\vapp) \leq bn^2(1-a)$. For a given value of $a$, this is maximized when $b$ is as large as possible.
 For Case 1 to hold, the largest possible value of $b$ is given by
 $a^2 = \frac{(a+b)^2}{2}$, which gives $b = a(\sqrt{2}-1)$.
 Therefore $\F(\vmap)-\F(\vapp) \leq \frac{n^2(\sqrt{2}-1)}{4} < \frac{n^2}{8}$, i.e.~$\F(\vapp) \geq \frac{8}{9}\F(\vmap)$.

 \noindent {\bf{Case 2:}} $a^2 \leq \frac{(a+b)^2}{2}$.
 This holds if $b \geq (\sqrt{2}-1)a$. Since $a+b \leq 1$, this is possible only if
 $a \leq 1/\sqrt{2}$.
 Now $\frac{\F(\vmap)-\F(\vapp)}{n^2} \leq  a^2 + b(1-a) - (a+b)^2/2 = \frac{a^2 - 4ab + 2b - b^2}{2}$.

 For a given $a$, this expression is quadratic in $b$ with a negative second derivative.
 This is maximized (by differentiating) for $b = 1 - 2a$.
 Since $b \leq a$, this value is possible only if $a \geq 1/3$.
 Similarly, for case 2 to hold with this value of $b$, we must have $a \leq \sqrt{2}-1$.
 Substituting this value of $b$, the difference in scores is $\frac{5a^2 - 4a + 1}{2}$. 

 Since this is quadratic with a positive second derivative, it is maximized when $a$ has 
 either the minimum or maximum possible value. For $a = 1/3$ this value is $1/9$,
 while for $a = \sqrt{2}-1$, it is $10-7\sqrt{2}$. In both cases, it is less than $1/8$.

     If $a \leq 1/3$ the maximum is achieved when $b = a$. In this case, the score difference
     is at most $(a - 2a^2)$ which is maximized for $a = 1/4$, where the value is $1/8$.
     (This is the worst case).

     For $\sqrt{2}-1 < a \leq 1/\sqrt{2}$, the maximum will occur for  $b = (\sqrt{2}-1)a$.
     Substituting this value for $b$, the score difference is $(\sqrt{2}-1)(a-a^2)$, which
     is maximized for $a = 1/2$, where its value is $(\sqrt{2}-1)/4  < 1/8$.
 \end{proof}
\end{theorem}
We believe that the bound for general $q$ is $\frac{4q}{4q+1}$.
This bound is not tight as for $q=1$ we have already shown that
the $\frac{4}{5}$ bound can be tightened to $\frac{13}{15}$.  With
$q=2$ we get a bound of $\frac{8}{9}$ which is better than
$\frac{13}{15}$.

\subsubsection*{Entropy potentials and the \alphapass\ algorithm}

As an aside, let us explore the behavior of \alphapass\ on another family of additive potentials ---
entropy potentials. Entropy potentials are of the form:
\begin{equation}
\vcp(\vv) = \lambda \sum_v n_v \log n_v, \mbox{ where }\lambda > 0
\end{equation}
The main reason \alphapass\ provides a good bound for Potts potentials is that it guarantees a
clique potential of at least $n_1^2$ where $n_1$ is the count of the most
dominant value in the optimal solution. 
The quadratic term compensates for possible sub-optimality of counts of other
values. If we had a
sub-quadratic term instead, say $n_1\log n_1$ for the entropy potentials, the same bound would not
have held. In fact the following theorem shows that for entropy potentials, even though \alphapass\
guarantees a clique potential of at least $n_1\log n_1$, that is not enough to
provide a good approximation ratio.

\begin{theorem}
\label{theorem:entropy}
\alphapass\ does not provide a bound better than $\frac{1}{2}$ for entropy potentials.
\begin{proof}
Consider a counter example where there are $R=n+\log n$ values. Divide the
values into two sets --- 
$A$ with $\log n$ values and $B$ with $n$ values. The vertex potentials are as follows: the vertices are
divided into $\log n$ chunks of size $n/\log n$ each. If the $j^{th}$ vertex lies in the $v^{th}$ chunk,
then let it have a vertex potential of $\log n$ with value $v$ in $A$ and a vertex potential of $\log n+\epsilon$
with the $j^{th}$ vertex in $B$. Let all other vertex potentials be zero. Also, let $\lambda=1$.

Consider the assignment which assigns the $v^{th}$ value in $A$ to the $v^{th}$ chunk. Its score is
$2n\log n - n\log\log n$. Now consider \alphapass, with $\alpha\in A$. Initially vertex $v$ will be
set to the $v^{th}$ value in $B$. The best assignment found by \alphapass\ will assign every vertex to
$\alpha$, for a total score of roughly $n+n\log n$. If $\alpha\in B$, then again
the best assignment
will assign everything to $\alpha$ for a total score of roughly $(n+1)\log n$.

Thus the bound is no better than $\frac{1}{2}$ as $n\rightarrow\infty$.
\end{proof}
\end{theorem}

Thus, \alphapass\ provides good approximations when the clique potential is
dominated by the most dominated value. We now look at \majority\ potentials,
which are linear in the counts $\{n_v\}_v$. Looking at
Theorem~\ref{theorem:entropy}, we expect that \alphapass\ will not have decent
approximation guarantees for \majority. This is indeed the case. We will prove
in Section~\ref{sec:majority} that neither \alphapass\ nor a natural
modification of \alphapass\ enjoy good approximation guarantees.

\subsection{Clique Inference for \Majority\ Potentials}
\label{sec:majority}

Recall that \Majority\ potentials are of the form $\vcp=f_a(\vn),\ a=\argmax_v\
n_v$. We consider 
linear majority potentials where $f_a(\vn) = \sum_{v}w_{av}n_{v}$. The matrix 
$W=\{w_{vv'}\}$ is not necessarily diagonally dominant or symmetric. 

We show that exact MAP for linear majority potentials can be found in
polynomial time.  We also present a modification to the \alphapass\
algorithm to serve as an efficient heuristic, but without approximation
guarantees. Then we present a Lagrangian relaxation based approximation, whose
runtime in practice is similar to \alphapass, but provides much better
solutions.
 
\subsubsection{Modified \alphapass\ algorithm}
In the case of linear majority potentials, we can incorporate the
clique term in the vertex potential, and this leads to the following
modifications to the \alphapass\ algorithm: (a) Sort the list for
$\alpha$ according to the modified metric
$\vnp_{i\alpha}+w_{\alpha\alpha}-\max_{v\ne\alpha}(\vnp_{iv}+w_{\alpha
v})$, and (b) While sweeping the list for $\alpha$, discard all
candidate solutions whose majority value is not $\alpha$.

However even after these modifications, \alphapass\ does not provide the same approximate guarantee
as for homogeneous Potts potentials, as we prove next. We denote a matrix $W$ as diagonally
dominant iff each of its diagonal entries are the largest in their corresponding rows.

\begin{theorem}
The modified \alphapass\ algorithm cannot have an approximation ratio better than $\frac{1}{2}$ on 
linear majority potentials with unconstrained $W$.
\begin{proof}
Consider the degenerate example where all vertex potentials are zero. Let $\beta$ and $\gamma$
be two fixed values and let the matrix $W$ be defined as follows: $w_{\beta\gamma}=M+\epsilon$, 
 $w_{\beta v}=M\ \forall v\neq \beta,\gamma$ and all the other entries in $W$ are zero.

In modified \alphapass, when $\alpha\neq\beta$, the assignment returned will have a zero score. When
$\alpha=\beta$, all vertices will prefer the value $\gamma$, so \alphapass\ will
have to assign exactly 
$n/2$ vertices as $\beta$ to make it the majority value, thus returning a score of $\frac{(M+\epsilon)n}{2}$. 
However, consider the assignment which assigns $n/R$ vertices to each value, with a score of $(R-1)Mn/R$.
Hence the approximation ratio cannot be better than $\frac{1}{2}$.
\end{proof}
\end{theorem}

\begin{theorem}
The modified \alphapass\ algorithm cannot provide an approximation bound better than $\frac{2}{3}$
for linear \majority\ potentials even when $W$ is diagonally dominant and each of its rows have equal sums.
\begin{proof}
The proof is by counter example which is constructed as follows. Let
the set of values $\range{\pr}$ be
divided into two subsets, $A$ and $B$ with $k$ and $n-k$ values respectively,
where $k < n/2$. Let $\beta\in B$ be a 
fixed value. We define $W$ as:
\begin{equation*}
w_{vv'} = \left\{ \begin{array}{ll}
n-k & v,v'\in A \\
k+1 & v\in A, v'=\beta \\
k+1 & v,v'\in B \\
0 & otherwise\end{array} \right. 
\end{equation*}

Thus $W$ is diagonally dominant with all rows summing to $(n-k)(k+1)$.

The vertex potentials $\vnp$ are defined as follows. Divide the vertices into $k$ chunks of size $n/k$ each. 
For the $v^{th}$ value in $A$, each vertex in the $v^{th}$ chunk has a vertex potential of
$2(n-k)$. Further, $\vnp_{i\beta}=2(n-k)\ \forall i$. The remaining vertex potentials are zero.

The optimal solution is obtained by assigning the $v^{th}$ value in $A$ to the $v^{th}$ chunk,
with a total score of $\frac{n}{k}.2(n-k).k + (n-k).\frac{n}{k}.k = 3n(n-k)$. 

In \alphapass, consider the pass for $\alpha\in A$. Each vertex $i$ prefers $\beta$ because 
$\vnp_{i\beta}+w_{\alpha\beta} = 3(n-k)$ is the best across all values. Thus, the best \alphapass\ 
generated assignment with majority value $\alpha$ is one where we assign $n/2$
vertices $\alpha$, including the 
$n/k$ vertices that correspond to the chunk of $\alpha$. The vertex and clique potentials of this assignment are
$\frac{n}{k}.2(n-k)+\frac{n}{2}.2(n-k)$ and $\frac{n}{2}(n-k)+\frac{n}{2}(n-k)$, giving a total
score of $2n(n-k)+\frac{2n(n-k)}{k}$. This gives an approximation ratio of
$\frac{2}{3}(1+\frac{1}{k})$.
 
Now consider the pass when $\alpha\in B$. Each vertex $i$ again prefers $\beta$ because 
$\vnp_{i\beta}+w_{\alpha\beta}=2n-k+1$ is the best across all values. 
If $\alpha=\beta$, the best \alphapass\ assignment is one where all vertices are assigned $\beta$, 
giving a total cost of $n(2n-k+1)$. If $\alpha\neq\beta$, then to make $\alpha$ the majority value, 
\alphapass\ can only output an assignment with score less than $n(2n-k+1)$. In this case, the
approximation ratio is no better than $\frac{2n-k+1}{3(n-k)}$.

Setting $k=\sqrt{n}$ and $n\rightarrow\infty$, we get the desired result.
\end{proof}
\end{theorem}
However, in practice where the $W$ matrix is typically sparse, our
experiments in Section~\ref{sec:expt-clique} show that \alphapass\ performs
well and is significantly more efficient than the exact algorithm
described next.

\subsubsection{Exact Algorithm}
Since \majority\ potentials are linear, we can pose the optimization problem in terms of Integer
Programs (IPs). 
Assume that we know the majority value $\alpha$. Then, the optimization problem corresponds to the
IP:
\begin{eqnarray}
\label{majority_ip1}
\max_{\vz}\sum_{i,v}(\vnp_{iv}+w_{\alpha v})z_{iv} \nonumber \\
\forall v:\ \sum_{i}z_{iv} \leq \sum_{i}z_{i\alpha}, \nonumber \\
\forall i:\ \sum_{v}z_{iv} = 1,\ z_{iv}\in \{0,1\}
\end{eqnarray}
We can solve $R$ such IPs by guessing various values as the majority value, and reporting the best
overall assignment as the output. However, Equation \ref{majority_ip1} cannot be relaxed to a linear
program. This can be easily shown by proving that the constraint matrix is not totally unimodular. 
Alternatively, here is a counter example: Consider a 3-node, 3-value graphical model with a zero $W$
matrix. Let the vertex potential vectors be $\vnp_0 = (1,4,0),\ \vnp_1=(4,0,4),\
\vnp_2=(3,4,0)$. While
solving for $\alpha=0$, the best IP assignment is $1,0,0$ with a score of $11$. However the LP
relaxation has the solution $\vz=(0,1,0;1,0,0;1/2,1/2,0)$ with a score of $11.5$.

This issue can be resolved by making the constraint matrix totally unimodular as follows. 
Guess the majority value $\alpha$, the count $k=n_\alpha$, and solve the following IP:
\begin{eqnarray}
\label{majority_ip2}
\max_{\vz}\sum_{i,v}(\vnp_{iv}+w_{\alpha v})z_{iv} \nonumber \\
\forall v\neq\alpha:\ \sum_{i}z_{iv} \leq k, \nonumber \\
\sum_{i}z_{i\alpha} =k, \nonumber \\
\forall i:\ \sum_{v}z_{iv} = 1,\ z_{iv}\in \{0,1\}
\end{eqnarray}
This IP solves the degree constrained bipartite matching problem, which can be solved exactly in polynomial time.
Indeed, it can be shown that the LP relaxation of this IP has an integral solution.
\begin{theorem}
The integer program in Equation \ref{majority_ip2} has a tight LP relaxation.
\begin{proof}
Denote the constraint matrix of program \ref{majority_ip2} by $A_{(m+n)\times mn}$, and let $A_1$
and $A_2$ denote its first $m-1$ and last $n+1$ rows respectively. The $n+1$
equality constraints can be converted into '$\leq$' constraints by adding negative slack variables.
For example, $s_i + \sum_v z_{iv} \leq 1$ and $s_\alpha + \sum_i z_{i\alpha} \leq k$. The variables are
now $(\vz,\vs)^T$, and the extra constraints are $\vs \leq \mathbf{0}$. The new constraint matrix of
this system (which has only inequality constraints) 
is given by $B = \left[\begin{array}{cc}A_1 & 0\\ A_2 & I\\ 0 & I \end{array}\right].$
The tightness of the LP relaxation follows if $B$ is totally unimodular. For that, it suffices to
prove that $A = \left[\begin{array}{cc}A_1\\A_2\end{array}\right]$ is totally unimodular. 
This is so because then $\left[\begin{array}{cc}A_1 & 0\\ A_2 & I\end{array}\right]$
would be totally unimodular, and by extension of the same argument, so would be $B$. The total
unimodularity of $A$ is proven as follows. 

Let $C$ be an arbitrary $t\times t$ sub-matrix of $A$. 
Our argument uses induction on $t$, with the base case $t=1$ being straightforward. Note that each
column in $A$ has exactly two 1's. Let $B_1$ denote the first $m$ rows and $B_2$ denote the
remaining $n$ rows of $A$.\\
{\bf{Case 1: }} $C$ has a column with all zeros. Then $det(C)=0$ and we are done. \\
{\bf{Case 2: }} $C$ lies totally inside either $B_1$ or $B_2$. Since there is only one non-zero entry in each
column of $B_1$ (or $B_2$), pick any column and $det(C)$ will be $\pm 1$ times the determinant of its 
$(t-1)\times (t-1)$ sub-matrix, depending on the column index. So using the induction hypothesis, we get $det(C)\in
\{0,\pm 1\}$. \\
{\bf{Case 3: }} $C$ spans rows in $B_1$ and $B_2$. Wlog assume that each column in $C$ has exactly
two 1's, otherwise we can apply the same argument as Case 1 or 2. 
Now, summing up the rows corresponding to $B_1$ and $B_2$ separately, we get 
$\forall j:\ \sum_{i\in rows(B_1)} c_{ij} = 1 = \sum_{i\in rows(B_2)} c_{ij}$. Hence the rows of $C$
are linearly dependent and so $det(C)=0$.
\end{proof}
\end{theorem}
Thus we can solve $O(Rn)$ such problems by varying $\alpha$ and $k$, and report the best solution.
We believe that since the subproblems are related, it should be possible to solve them incrementally 
using combinatorial approaches.

\newcommand{\lam}{{\gamma}}

\subsubsection{Lagrangian Relaxation based Algorithm for \Majority\ Potentials}
\label{sec:algoLR}

 Solving the linear
system in Equation \ref{majority_ip2} is very expensive because we
need to solve $O(Rn)$ LPs, whereas the system in Equation \ref{majority_ip1}
cannot be solved exactly using a linear relaxation.  Here, we look at a
Lagrangian Relaxation based approach, where we solve the system in Equation
\ref{majority_ip1} but bypass the troublesome constraint $\forall v\neq\alpha:
\sum_i z_{iv}\leq \sum_i z_{i\alpha}$.

\ignore{ If the majority label $\alpha$ is known apriori, then the vertex
and clique potentials can be merged into a single term $\phi^\alpha_u(y)
= \phi_u(y) + w_{\alpha y}$. We can therefore compute a vertex optimal
assignment under these new vertex potentials, while maintaining the constraint
that the majority label in the assignment be $\alpha$. This is achieved
by solving the following integer program:

\begin{eqnarray} \label{eqn:maj_obj} \max_{\vx}
\sum_{u,y}(\np_{uy}+w_{\alpha y})x_{uy} \nonumber \\ \forall y: \sum_u
x_{uy} \leq \sum_u x_{u\alpha} \\ \forall u: \sum_y x_{uy} = 1,~~x_{uy}\in
\{0,1\} \end{eqnarray}

Since the majority label $\alpha$ is not known a priori, we can solve this
integer program for all $\alpha$ and report the best assignment overall.
However, there are two issues with this approach. First, this integer
program does not have a tight relaxation to a linear program. This is
can be shown by illustrating that the constraint matrix is not always
totally unimodular. Consequently, we cannot use LP solvers to optimize
this integer program.

Secondly, any assignment that optimizes the new vertex potentials
$\phi^\alpha$ has to ensure the majority of $\alpha$. This is analogous
to computing a degree-constrained matching in a bipartite graph,
which, although polynomially solvable, is still expensive to compute
\footnote{General algorithms to compute degree constrained matchings
have a runtime of $O((n+m)^{2.5}$.}.  }

We make use of the Lagrangian Relaxation technique to move the
troublesome majority constraint to the objective function. Any
violation of this constraint is penalized by a positive penalty
term. Consider the following modified program, also called the
Lagrangian: 
\begin{eqnarray} \label{eqn:maj_lagrangian} 
L(\vl) = L(\lam_1,\ldots,\lam_R) = \max_{\vz} \sum_{i,v}(\vnp_{iv}+w_{\alpha v})z_{iv} &+&
\sum_v\lam_v(\sum_i
z_{i\alpha} - \sum_i z_{iv} )\nonumber \\ \forall i: \sum_v z_{iv} =
1&,&~~z_{iv}\in \{0,1\} 
\end{eqnarray} 

For $\vl\geq {\bf 0}$, and feasible
assignments $\vz$, $L(\vl)$ is an upper bound for our objective in Equation
\ref{majority_ip1}. Thus, we compute the lowest such upper
bound: 
\begin{equation} \label{eqn:maj_lag_obj} 
L^* = \min_{\vl\geq{\bf 0}} L(\vl) 
\end{equation} 

Since the penalty term in Equation~\ref{eqn:maj_lagrangian} is linear
in $\vz$, we can merge it with the first term 
to get another set of modified vertex potentials: 
\begin{equation}
\label{eqn:mod_np} 
\vnp^\alpha_{iv} \triangleq \vnp_{iv} + w_{\alpha v} - \lam_v
+ \left\{ \begin{array}{ll} \sum_{v'}\lam_{v'} & v=\alpha
\\ 0 & v\neq\alpha \end{array}\right.  
\end{equation} 

Equation \ref{eqn:maj_lagrangian} can now be rewritten in terms of $\vnp^\alpha$,
with the only constraint that $\vz$ be a assignment:
\begin{eqnarray}
\max_{\vz} \sum_{i,v}\vnp^\alpha_{iv}z_{iv} \nonumber \\
\forall i:\sum_v z_{iv}=1,\ \ z_{iv}\in \{0,1\} 
\end{eqnarray} 

 Hence, $L(\vl)$ can
be computed by independently assigning each vertex $i$ to its best value
, viz.~$\argmax_v \vnp^\alpha_{iv}$.

We now focus on computing $L^*$. We use an iterative approach, beginning
with $\vl={\bf{0}}$, and carefully choose a new $\vl$ at each step to
get a non-increasing sequence of $L(\vl)$'s. We describe the method of
choosing a new $\vl$ later in this section, and instead outline sufficient
conditions for termination and detection of optimality.

\begin{theorem} 
\label{theorem:lr}
$\vz^*$ and $\vl^*$ are optimum solutions to Equations
\ref{majority_ip1} and \ref{eqn:maj_lag_obj} respectively if they
satisfy the conditions: 
\begin{eqnarray} 
\forall v: &&\sum_i z^*_{iv} \leq \sum_i z^*_{i\alpha} \label{eqn:violation1}\\ 
\forall v: &&|\lam^*_y(\sum_i z^*_{iv} -
\sum_i z^*_{i\alpha})| = 0 \label{eqn:violation2}
\end{eqnarray}
\end{theorem} 

Theorem \ref{theorem:lr} holds only for fractional $\vz^*$. To see how, consider
an example with 3 vertices and 2
values. Let $\vnp_{i0}+w_{\alpha 0} > \vnp_{i1}+w_{\alpha 1}$ for all $i$ and
$\alpha$. During Lagrangian relaxation with
$\alpha=1$, initially $\vl={\mathbf{0}}$ will cause all vertices to be assigned
value $0$, violating Equation 
\ref{eqn:violation1}. Since the count difference $\sum_i z_{i0} - \sum_i z_{i1} \in \{\pm 1,\pm
2,\pm 3\}$, any non-zero $\lam_0$ will violate Equation \ref{eqn:violation2}. 
Subsequent reduction of $\lam_0$ to zero will again cause the original violation of
Equation \ref{eqn:violation1}. Consequently,
one of Equations \ref{eqn:violation1} and \ref{eqn:violation2} 
will never be satisfied and the algorithm will oscillate.
\ignore{
In practice, the second condition is relaxed to $|\lam_v(\sum_u x^*_{uy} -
\sum_u x^*_{u\alpha})| \leq \epsilon$ to achieve convergence with
tolerable errors. However the integrality of $\vx$ can cause problems
with convergence. The conditions are achievable only by fractional
$\vx$, and it is possible that we may never satisfy these conditions, even
though our $L(\vl)$'s may be non-increasing. To tackle this, we stop
the execution of the algorithm after a sufficient number of iterations.
}

To tackle this, we relax Equation \ref{eqn:violation2} to $|\lam_v(\sum_i
z^*_{iv} - \sum_i z^*_{i\alpha})| \leq \epsilon$, where $\epsilon$ is a small fraction of an upper bound on
$\lam_v$, whose computation is illustrated later. This helps in reporting assignments that
respect the majority constraint in Equation \ref{eqn:violation1} and are close to the optimal.

The outline of the algorithm is described in Figure
\ref{fig:lag_algo}. We now discuss a few possible approaches to select a new
$\vl$ at every step.

\subsubsection*{Subgradient Optimization}
This approach can be used to change all components of $\vl$ in a single step. Subgradient
optimization generates a sequence of direction vectors $\{\vd^1,\vd^2,\ldots\}$ and positive step sizes 
$\{\eta_1,\eta_2,\ldots\}$. At the $k^{th}$ step, the $\vl$ vector is changed as:
\begin{equation}
\label{eqn:subgrad}
\lam_v \leftarrow \max(0,\lam_v + \eta_k\vd^k)
\end{equation}
In its simplest form, the direction $\vd^k$ is the violation vector $(\sum_i
z_{iv} - \sum_i
z_{i\alpha})_{v=1\ldots R}$. Thus, if a value $v$ has a count greater than $\alpha$, then
$\lam_v$ will be increased to take some vertices away from $v$. In practice though, $\vd^k$ is
usually a convex combination of the violation vector and the previous direction $\vd^{k-1}$. This
helps in avoiding oscillations of the kind where $\vl^{k+2}$ is very close to $\vl^k$, while
simultaneously moving closer to the optimum.

The subgradient optimization framework allows various ways to choose the step sizes. For example,
if we choose to set the direction using only the violation vector, then a
	sequence $\{\eta_1,\eta_2,\ldots\}$
of step sizes satisfying (i) $\lim_{k\rightarrow\infty}\eta_k = 0$ and 
(ii) $\sum_{k=0}^{\infty}\eta_k = \infty$ will ensure asymptotic convergence \cite{Guta03Subgradient}.
These are not the only set of sufficient conditions that
guarantee convergence. Practical implementations often compute $\eta_k$ using the current value of
$L(\vl)$, current violations, and a few user defined parameters.

However, during experimentation the degrees of freedom in choosing the step sizes posed a big
problem for us. Since a single step size is shared across all components of $\vl$, a large step size
moved everything to $\alpha$, while a small step size considerably slowed down convergence. Data
independent approaches to choosing step sizes also failed to converge in a reasonable number of
iterations. In general, subgradient optimization is known to require very careful tweaking of 
the step sizes across iterations in order to achieve meaningful convergence speeds
\cite{Guta03Subgradient}.
For this reason, we looked at alternate approaches to change $\vl$.

\subsubsection*{Golden Search based Coordinate Descent} If all components
of $\vl$ except one (say $\lam_v$), are kept fixed, then $L(\vl)$ is
a quasi-convex function of $\lam_v$.  Thus, it has a unique global
minima, which can be found using golden search, which is an efficient
line search method. We choose the value $v$, whose corresponding 
violation is the highest in magnitude.

Golden search requires lower and upper bounds on $\lam_v$ and evaluates
$L(\vl)$ at various $\vl$'s inside that interval. As before, $L(\vl)$
can be easily obtained by a computing an assignment which is vertex-optimal
wrt the $\vnp^\alpha$'s. We use the trivial lower bound of zero,
and estimate a good upper bound from the current solution state. If
currently $\sum_i z_{iv}\leq \sum_u z_{u\alpha}$, then $\lam_v$ (which is a penalty
parameter) can be decreased, and therefore the current value of
$\lam_v$ can serve as an upper bound.  On the other hand, if we start
increasing $\lam_v$, then one by one, the vertices currently assigned $v$
will switch to their next best values, and by a particular increased
value of $\lam_v$, all vertices assigned $v$ would have flipped. There is no 
need to increase $\lam_v$ beyond this point, so we
use this value of $\lam_v$ as our upper bound, which can thus
be summarized as : 
\begin{equation*} 
UB(v) = \max_{i:z_{iv}=1} \delta_i
\end{equation*} 
where (denoting the current second best value of $i$
by $\beta$), 
\begin{equation*} \delta_i =\left\{ \begin{array}{ll}
\vnp_{iv}+w_{\alpha v}-\vnp_{i\beta}-w_{\alpha\beta}+\lam_\beta
& \beta\neq\alpha \\ \frac{1}{2}(\vnp_{iv}+w_{\alpha
v}-\vnp_{i\alpha}-w_{\alpha\alpha}-\sum_{\beta\neq v}\lam_v) &
\beta=\alpha \end{array}\right.  \end{equation*}

In spite of its simplicity, Golden Search suffers from a few drawbacks
that come to light during experimentation. First, it always performs
$\theta(\log UB(v))$ evaluations of $L(\vl)$. This can drive up the
overall runtime of the algorithm. Second, a change in $\lam_v$
affects the modified vertex potentials for values $v$ and $\alpha$
(Equation \ref{eqn:mod_np}). Thus, a large change in $\lam_v$ may
flip many vertices to $\alpha$, causing a big change in the current assignment
$\vz$, and we may end up spending the next few iterations repairing these
changes. Third, line search methods zero-in on the optima by evaluating
the objective at various points and choosing a sub-interval accordingly. In
our case, due to the integrality of $\vz$, ties can happen at many places
in an interval. In such a scenario, arbitrary tie resolution may cause
a wrong sub-interval to be chosen for further consideration.

To tackle these issues, we use a more conservative coordinate descent
approach, which we describe next and also use in our experiments.

\subsubsection*{Conservative Coordinate Descent}
We can avoid a large number of flips in the current assignment if we replace our golden search method
with a more conservative one. Let $v$ be the worst violating value in the current iteration. We will first 
consider the case when its count exceeds that of $\alpha$, so that Equation \ref{eqn:violation1} does not hold. 

To decrease the count of $v$, we need to increase $\lam_v$. 
Let $i$ be a vertex currently assigned $v$ and let $\beta(i)$ be its second most
preferred value
under the vertex potentials $\vnp^\alpha_i$.
The vertex $j = \argmax_{i:z_{iv}=1} \vnp^\alpha_{i\beta(i)}-\vnp^\alpha_{iv}$ is the easiest to
flip. So we increase $\lam_v$ till the point when this difference becomes zero. The new value of 
$\lam_v$ is therefore given by:
\begin{equation}
\label{eqn:consCase1}
\lam_v = \min_{i:z_{iv}=1} \left\{\begin{array}{ll}
\Delta\vnp(i,v,\beta(i))+\lam_{\beta(i)} & \beta(i)\neq\alpha \\
\frac{1}{2}(\Delta\vnp(i,v,\alpha)-\sum_{v'\neq v}\lam_{v'}) & \beta(i)=\alpha
\end{array}\right.
\end{equation}
where $\Delta\vnp(i,v,v')$ denotes $\vnp_{iv}+w_{\alpha v}-\vnp_{iv'}-w_{\alpha
v'}$. It is possible
that by flipping vertex $j$, $\beta(j)$ now violates Equation \ref{eqn:violation1}. Further,
increasing $\lam_v$ also increases $\vnp^\alpha_{i\alpha}$, so some other
vertices that are not assigned $v$
may also move to $\alpha$. However since the change is conservative, we expect this behavior to be
limited. In our experiments, we found that this conservative scheme converges much faster than
golden section over a variety of data.

We now look at the case when Equation \ref{eqn:violation1} is satisfied by all
values but 
Equation \ref{eqn:violation2} is violated by some value $v$. In this
scenario, we need to decrease $\lam_v$ to decrease the magnitude of the violation. Here too, we
conservatively decrease $\lam_v$ barely enough to flip one vertex to $v$.
If $i$ is any vertex not assigned value $v$ and $\beta(i)$ is its current value, then the new value of
$\lam_v$ is given by:
\begin{equation}
\label{eqn:consCase2}
\lam_v = \max_{i:z_{iv}\neq 1} \left\{\begin{array}{ll}
\Delta\vnp(i,v,\beta(i))+\lam_{\beta(i)} & \beta(i)\neq\alpha \\
\frac{1}{2}(\Delta\vnp(i,v,\alpha)-\sum_{v'\neq v}\lam_{v'}) & \beta(i)=\alpha
\end{array}\right.
\end{equation}
Note that the arguments of Equations \ref{eqn:consCase1} and \ref{eqn:consCase2} are the same. In this case
too, in spite of a conservative move, more than one vertex marked $\alpha$ may
flip to some other value,
although at most one of them will be flipped to $v$. As before, the small magnitude of the change
restricts this behavior in practice.

\SetKw{Kw}{break}

\begin{algorithm}
\KwIn{$\vnp,W,\alpha$,maxIters,tolerance}
\KwOut{approximately best assignment $\vapp$}
$\vl\leftarrow {\mathbf{0}}$\;
iter $\leftarrow$ 0\;
$\vzh\leftarrow\mbox{Assignment with all vertices assigned }\alpha$\;
\While{iter $<$ maxIters} {
	Compute $L(\vl)\ \  \mbox{(Equation \ref{eqn:maj_lagrangian})}$, let $\vz$
	be the solution\;
	\If {$\F(\vz)>\F(\vzh)$} {
		$\vzh\leftarrow\vz$\;
	}
	$(v,\Delta)\leftarrow$ Worst violator and violation (Equations \ref{eqn:violation1} and \ref{eqn:violation2})\;
	\uIf {$\Delta\ <$ tolerance} {
		\mbox{We are done, $L^*=L(\vl)$}\;
		\Kw{break}\;
	} 
	\uElseIf {using subgradient optimization} {
		Compute the new direction $\vd^{iter}$ and step-size $\eta^{iter}$\;
		Modify $\vl$ using Equation \ref{eqn:subgrad}\;
	}
	\Else {
		Modify $\lam_v$ using golden search or conservative descent\;
	}
	iter $\leftarrow$ iter+1\;
}
Construct value assignment $\vapp$ from $\vzh$\;
\KwRet{$\vapp$}
\caption{Compute $L^*$}
\label{fig:lag_algo}
\end{algorithm}

\section{Applications and Experiments}
\label{sec:expts}
We present results of three different experiments. 

First, in Section~\ref{sec:expt-clique} we compare our clique inference algorithms against applicable
alternatives in the literature. We compare the algorithms on speed and accuracy
of the output assignments. For \Potts\ potentials, we show that
\alphapass\ is superior to the TRW-S and min-cut based algorithms. For
\majority\ potentials, we compare the modified \alphapass\ and Lagrangian
relaxation based algorithms against the exact LP-based approach and the iterated
conditional modes (ICM) algorithm. 

Second, in Section~\ref{sec-domain} we demonstrate the application of the
generalized collective framework on \app\ and show that using a good set of
properties can bring down the test error significantly. 
Finally, in Section~\ref{sec-word} we show that message passing on the cluster
graph is a more effective  way to perform inference compared to
alternatives such as ordinary belief propagation, and enjoys better convergence
speeds.

\subsection{Clique Inference Experiments}
\label{sec:expt-clique}

In this section, we compare our algorithms against sequential tree
re-weighted message passing (TRW-S) and graph-cut based inference
\ignore{\cite{boykov01Fast}} for clique potentials that are decomposable over clique edges; 
and with ICM when the clique potentials are not edge decomposable.  We
compare them on running time and quality of the MAP.  
Our experiments were performed on both synthetic and real data.

\noindent{\textbf{Synthetic Dataset:}} We generated cliques with $100$ vertices
and $R=24$ values
by choosing vertex potentials at random from $[0,2]$ for all
values.
\ignore{
    Different versions of the dataset were created by choosing
from various families of clique potentials. }
A Potts version (\synpair)
was created by gradually varying $\lambda$ \ignore{in $[0,2]$}, and
generating $25$ cliques for every value of $\lambda$.  We also created
analogous \synentropy, \synms\ and \synmsq\ versions of the dataset by
choosing entropy, linear makespan ($\lambda\max_v n_v$) and square
makespan ($\lambda\max_v n_v^2$) clique potentials respectively. 

For \majority\ potentials we generated two kinds of datasets (parameterized by $\lambda$):
(a) \synmajd\ obtained by generating a random symmetric $W$ for each clique, 
where $W_{vv}=\lambda$ was the same for all $v$
and $W_{vv'}\in [0,2\lambda]\ (v\neq v')$, and (b) \synmajs\ 
from symmetric $W$ with $W_{ij}\in [0,2\lambda]$ for all $i,j$, roughly $70$\% of whose entries were zeroed.

Of these, only \synpair\ is decomposable over clique edges. 

\noindent{\textbf{CoNLL Dataset:}} The CoNLL 2003 dataset\footnote{\url{http://cnts.uia.ac.be/conll2003/ner/}} 
is a popular choice for demonstrating the benefit of collective labeling in named entity recognition tasks.  
We used the BIOU encoding of the entities, that resulted in $20$ labels.
We took a subset of $1460$ records from the test set of CoNLL, 
\ignore{
and obtained a set of $2235$ cliques using the representation described in Section \ref{sec:app-ie}. From this set, we }
and selected all $233$ cliques of size $10$ and above. The median and largest clique sizes were $16$ and  $259$ respectively.
The vertex potentials of the cliques were set by a sequential Conditional Random Field trained on a
separate training set. We created a Potts version by setting $\lambda=0.9/n$, where $n$ is the clique size. Such a 
$\lambda$ allowed us to balance the vertex and clique potentials for each clique. A majority version
was also created by learning $W$ discriminatively in the training phase. \ignore{The $W$ thus learnt was 85\% sparse.}

All our algorithms were written in Java. We compared these with C++ implementations of 
the TRW-S\footnote{\url{http://www.adastral.ucl.ac.uk/~vladkolm/papers/TRW-S.html}},
and graph-cut based expansion
algorithms\footnote{\url{http://vision.middlebury.edu/MRF/}}~\cite{boykov01Fast,Szeliski06Comparative,kolmogorov04What,boykov04Experiment}. 
All experiments were performed on a Pentium-IV 
3.0 GHz machine with four processors and 2 GB of RAM. 

\begin{figure*}
\begin{center}
\subfigure[Comparison with TRW-S:
Synthetic]{\label{fig:trws1synth}\includegraphics[width=0.37\textwidth]{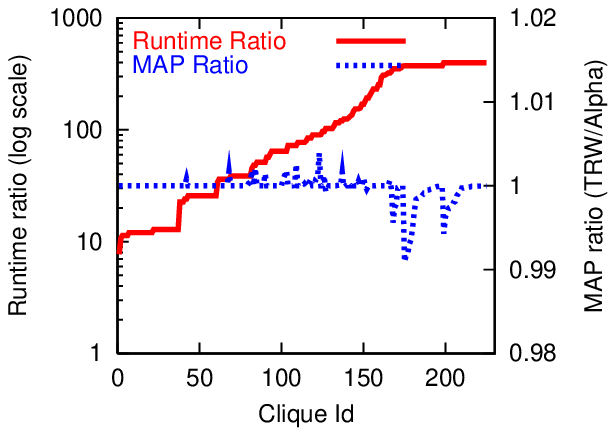}}
\subfigure[Comparison with TRW-S:
CoNLL]{\label{fig:trws1conll}\includegraphics[width=0.37\textwidth]{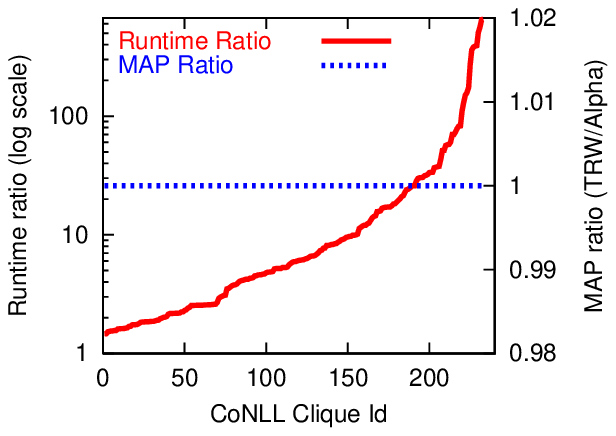}}
\subfigure[Comparison with GraphCut:
Synthetic]{\label{fig:mincut1synth}\includegraphics[width=0.37\textwidth]{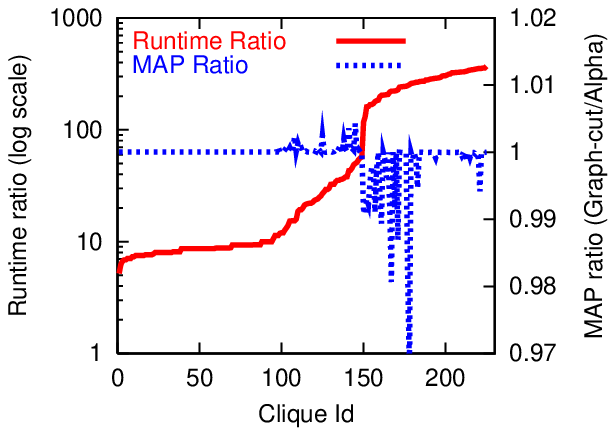}}
\subfigure[Comparison with ICM:
Synthetic]{\label{fig:icm}\includegraphics[width=0.37\textwidth]{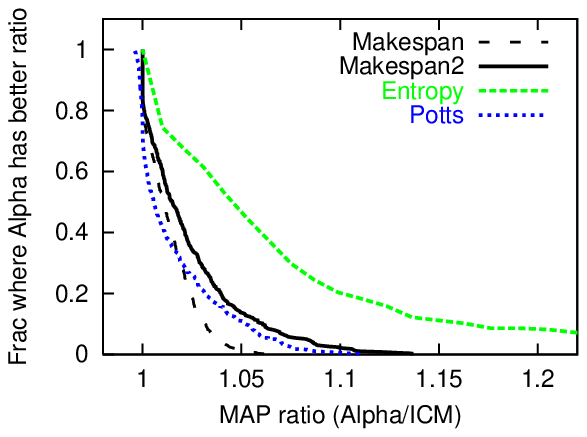}}
\caption{\label{fig:cmp}Comparison with TRW-S, Graph-cut and ICM}
\end{center}
\end{figure*}

\begin{figure*}
\begin{center}
\subfigure[\synmajd]{\label{fig:majdense}\includegraphics[width=0.37\textwidth]{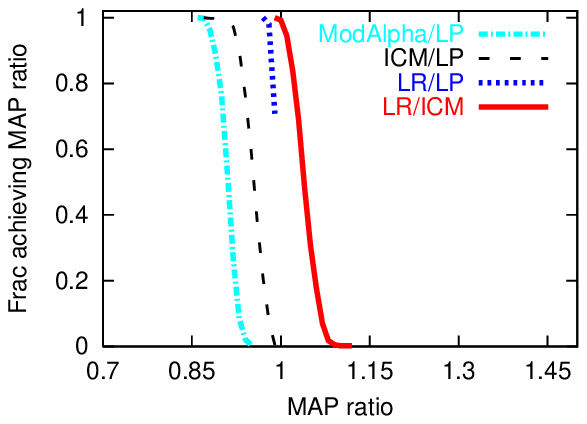}}
\subfigure[\synmajs]{\label{fig:majsparse}\includegraphics[width=0.37\textwidth]{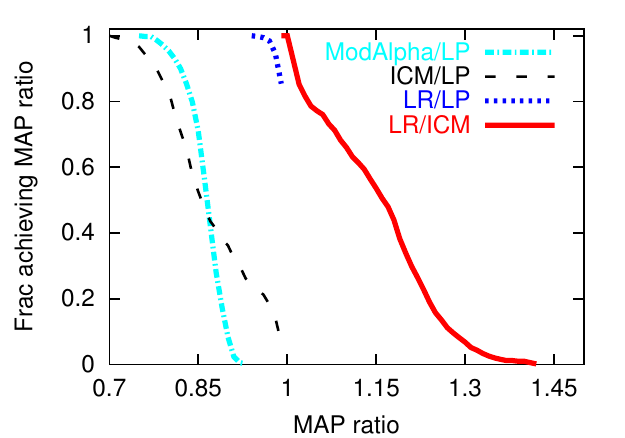}}
\subfigure[CoNLL]{\label{fig:majconll}\includegraphics[width=0.37\textwidth]{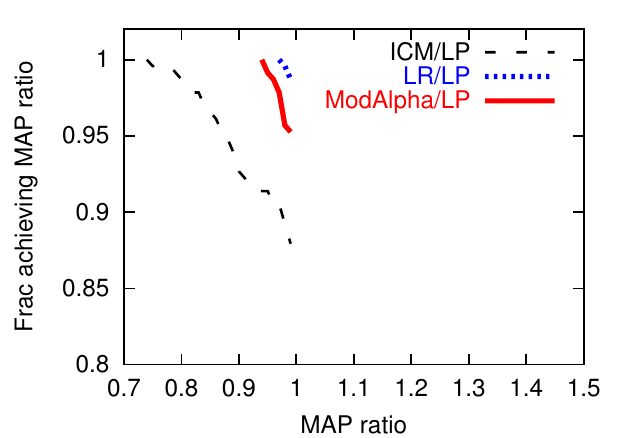}}
\subfigure[Time(CoNLL)]{\label{fig:majconllTime}\includegraphics[width=0.37\textwidth]{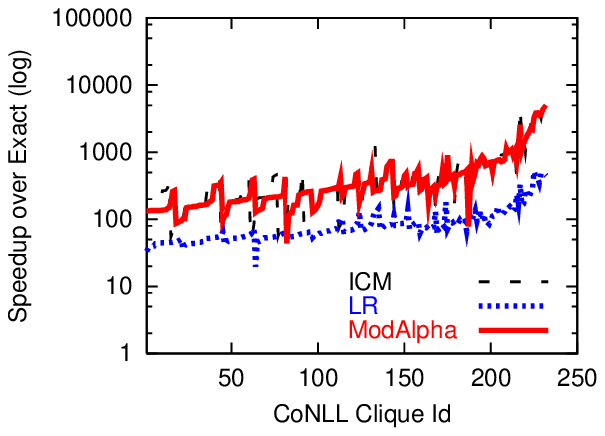}}
\caption{\label{fig:majcmp}Comparing AlphaPass, ICM, Lagrangian Relaxation and
Exact on \majority\ potentials}
\end{center}
\end{figure*}
\subsubsection{Edge decomposable potentials}
\ignore{
TRW-S is the state-of-the-art algorithm for inference, and has been
proved to provide convergent upper bounds on the MAP.  Other inference
algorithms for associate pairwise potentials include Boykov et.~al.'s
algorithm \cite{boykov01Fast}. However, the algorithm does not fully
exploit the special nature of the Gini potentials, and its existing
version is cubic in the clique size.}
Figures \ref{fig:trws1synth} and \ref{fig:trws1conll} compare the 
performance of TRW-S vs \alphapass\ on the two datasets. In Figure 
\ref{fig:trws1synth}, we varied $\lambda$  uniformly  in $[0.8,1.2]$ with 
increments of $0.05$.
This range of
$\lambda$ is of special interest, because it allows maximal contention
between the clique and vertex potentials. For $\lambda$ outside this
range, the MAP is almost always a trivial assignment, viz.~one which
individually assigns each vertex to its best value, or assigns all vertices
to a single value.

We compare two metrics --- (a) the quality of the MAP score, captured
by the ratio of the TRW-S MAP score with the \alphapass\ MAP score, and
(b) the runtime required to report that MAP, again as a ratio.  Figure
\ref{fig:trws1synth} shows that while both the approaches report
almost similar MAP scores, the TRW-S algorithm is more than $10$ times
slower in more than $80\%$ of the cases, and is never faster. This is
expected because each iteration of TRW-S costs $O(n^2)$, and multiple
iterations must be undertaken. In terms of absolute run times, a single
iteration of TRW-S took an average of $193$ms across all cliques in
\synpair, whereas our algorithm returned the MAP in $27.6\pm8.7$ms.
Similar behavior can be observed on CoNLL dataset in Figure
\ref{fig:trws1conll}.  Though the degradation is not as much as
before, mainly because of the smaller average clique size, 
TRW-S is more than $5$ times slower on more than half the
cliques.

Figure~\ref{fig:mincut1synth} shows the comparison with Graph-cut
based expansion.  The MAP ratio is even more in favor of
\alphapass, while the blowup in running time is of the same order 
of magnitude as TRW-S. This is surprising because based on the
experiments in~\cite{Szeliski06Comparative} we expected this method to
be faster.  One reason could be that their experiments were on grid
graphs whereas ours are on cliques.

\subsubsection{Non-decomposable potentials}
In this case, we cannot compare against the TRW-S or graph-cut based algorithms.
Hence we compare with the ICM algorithm that
has been popular in such scenarios~\cite{lu03Link}.  We varied
$\lambda$ with increments of $0.02$ in $[0.7,1.1)$ and generated $500$
cliques each from
\synpair, \synmajd, \synmajs, \synentropy, \synms\ and \synmsq. 
We measure the ratio of MAP score of \alphapass\ with ICM and for each ratio
$r$ we plot the fraction of cliques where \alphapass\ returns a MAP that
results in a ratio better than $r$.  Figure \ref{fig:icm} shows the
results on all the potentials except majority. The curves for linear and square
makespan lie totally to the right of $ratio=1$, which is expected
because \alphapass\ will always return the optimal answers for those
potentials. For Potts too, \alphapass\ is better than ICM for
almost all the cases. For entropy, \alphapass\ was found to be significantly better than 
ICM in all the cases.  The runtimes of ICM and \alphapass\ were similar.

\subsubsection*{Majority Potentials}

In Figures \ref{fig:majdense} and \ref{fig:majsparse}, we compare ICM, Lagrangian Relaxation
(LR) and modified \alphapass\ 
with the LP-based exact method on synthetic data. The dotted curves plot, 
for each MAP ratio $r$, the fraction of cliques on which ICM (or LR or modified \alphapass) returns a MAP score
better than $r$ times the true MAP. The solid curve plots the fraction of cliques where LR 
returns a MAP score better than $r$ times the ICM MAP. On \synmajd, both modified \alphapass\ and
ICM return a MAP score better than $0.85$ of the true MAP, with ICM being slightly better. However,
LR outperforms both of them, providing a MAP ratio always better than 0.97 and returning the true
MAP in more than 70\% of the cases.
In \synmajs\ too, \ignore{modified \alphapass\ and ICM dominate on roughly half the cliques each, but again,} LR
dominates the other two algorithms, returning the true MAP in more than 80\% of the cases, 
with a MAP ratio always better than $0.92$. The solid 
curve in Figure \ref{fig:majsparse} shows that on average, LR returns a MAP score $1.15$ times that
of ICM. Thus, LR performs much better than its competitors across dense as well as sparse majority potentials.

\ignore{
Figure \ref{fig:majconll} displays similar results on the CoNLL dataset, whose $W$ matrix is 85\% sparse. 
ICM performs poorly, returning the true MAP in only 10\% of the cases across all clique sizes, 
and achieving an average MAP ratio
of 0.68 against the exact method. On the other hand, \alphapass\ returns the true MAP in more than 
40\% of the cases, with an excellent average MAP ratio of 0.98, and almost always provides a
solution better than ICM.
}
The results on CoNLL dataset, whose $W$ matrix is 85\% sparse, are displayed in Figure \ref{fig:majconll}.
ICM, modified \alphapass\ and LR return the true MAP in 87\%, 95\% and 99\% of the cliques
respectively, with the worst case MAP ratio of LR being $0.97$ as opposed to $0.94$ and $0.74$ for
modified \alphapass\ and ICM respectively.

Figure \ref{fig:majconllTime} displays runtime ratios on all CoNLL cliques for all three inexact
algorithms. ICM and modified \alphapass\ are roughly $100$-$10000$ times faster than the 
exact algorithm, while LR is roughly twice as expensive as ICM and modified \alphapass.
\ignore{
Barring a few highly expensive outliers, and ignoring the dependence on $m$, it appears that the 
exact method is roughly $O(n)$ times slower than \alphapass\.} Thus, for practical majority 
potentials, LR and modified \alphapass\ seem to quickly provide highly accurate solutions.

From now on, while doing the top-level message passing on the cluster graph, we
shall use the
\alphapass\ and Lagrangian-relaxation based algorithms for computing messages
from a clique, in the presence of \Potts\ and \majority\ potentials
respectively.

\subsection{\appcaps}
\label{sec-domain}
We now move on the generalized collective inference framework, and show that a
good set of properties can help us in \app.
We focus on the bibliographic task, where the aim is to adapt a sequential model across widely varying publications pages of authors.
Our dataset consists of $433$ bibliographic entries from the web-pages
of $31$ authors, hand-labeled with $14$ labels such as Title, Author, Venue,
Location and Year.
Bibliographic entries across different authors differ in 
various aspects like label-ordering, missing labels, punctuation, HTML
formatting and bibliographic style.

A fraction of $31$ \domain s were used to train a baseline sequential model. 
The model was trained with the LARank algorithm of~\cite{bordes07}, using the BCE encoding 
for the labels. 
We used standard extraction features in a window around each token, along with label transition features~\cite{peng04}.

For our collective framework, we use the following decomposable properties: 
\begin{eqnarray*}
\pr_1(\vx,\vy) &=& \text{First non-Other label in }\vy \\
\pr_2(\vx,\vy) &=& \text{Token before the Title segment in }\vy \\
\pr_3(\vx,\vy) &=& \text{First non-Other label after Title in }\vy \\
\pr_4(\vx,\vy) &=& \text{First non-Other label after Venue in }\vy\\
\end{eqnarray*}
Inside a \domain, any
one of the above properties will predominantly favor one value, e.g.~$\pr_3$ 
might favor the value `Author' in one domain, and `Date' in another.
 Thus these properties encourage consistent labeling around the Title and Venue segments.
We use \Potts\ potential for each property, with $\lambda=1$.

Some of these properties, e.g.~$\pr_3$, operate on non-adjacent labels, and thus are not
Markovian. This can be easily rectified by making 'Other' an extension of its
predecessor label, e.g.~an 'Other' segment after 'Title'  can be
relabeled as 'After-Title'.


The performance results of the collective model with the above properties versus the
baseline model are presented in Table
\ref{tab:domain}. For the test \domain s, we report token-F1 of the important labels 
--- Title, Author and Venue. The accuracies are averaged
over five trials. The collective model leads to upto 25\% reduction in the test error for Venue
and Title, labels for which we had defined related properties. The gain is statistically significant (p $<$ 0.05).
 The improvement is more prominent when only a few \domain s are available for training. 
\begin{table}[h] 
\begin{center} 
\begin{tabular}{|c|cc|cc|cc|} \hline 
&\multicolumn{2}{|c|}{{\bf Title}}&\multicolumn{2}{|c|}{{\bf
Venue}}&\multicolumn{2}{|c|}{{\bf Author}}\\
Train \%&Base&CI&Base&CI&Base&CI\\
5&70.7&74.8&58.8&62.5&74.1&74.3\\
10&78.0&82.1&69.2&72.2&75.6&75.9\\
20&85.8&88.6&76.7&78.9&80.7&80.7\\
30&91.7&93.0&81.5&82.6&87.7&88.0\\
50&92.3&94.2&83.5&84.5&89.4&90.0\\\hline
\end{tabular} 
\end{center} 
\caption{ \label{tab:domain}
Token-F1 of the Collective and Base Models on the \appcaps\ Task}
\end{table} 
Figure~\ref{fig:perDomain} shows the error reduction on individual test \domain s  for one particular split
when five \domain s were used for training and $26$ for testing. The errors are computed from the
combined token F1 scores of Title, Venue and Author. For some \domain s \ignore{(like
5,17,24),} the errors are reduced by more than 50\%. Collective inference
increases errors in only two \domain s.

Finally, we mention that for this task, applying the classical collective inference setup with
cliques over repeated occurrences of words leads to very minor gains. In this
context, the generalized collective inference framework is indeed a 
much more accurate mechanism for joint labeling.

\begin{figure}
\begin{center}
\includegraphics[width=0.45\textwidth]{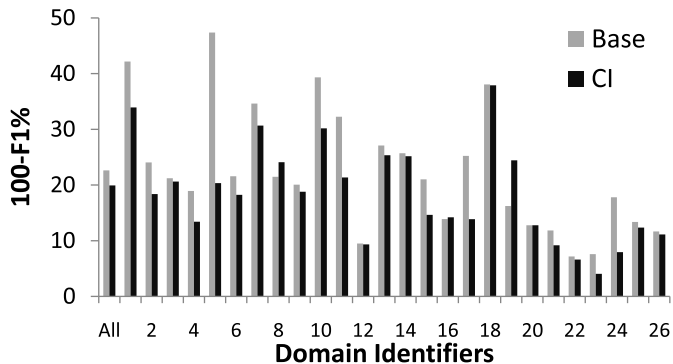}
\vspace*{0.2in}
\caption{\label{fig:perDomain}Per-\domain\ error for the base and collective inference (CI) model}
\end{center}
\end{figure}

\subsection{\label{sec-word}Collective Labeling of Repeated Words}
We now establish that even for simple collective inference setups without any
multi-clique properties, message passing on the cluster graph (abbreviated as CI) is a better option.
We consider information extraction over text records, and define cliques over multiple occurrences of words.
\ignore{
In this experiment, we establish that \mpc\ is a more
viable approach than other alternatives. We perform collective inference with
cliques on repeated occurrences of words. } We create two versions of the
experiment --- with \Potts\ and \Majority\ potentials on the cliques respectively.

Since the \Potts\ potential is decomposable over the clique edges, we compare
\mpc\ 
\ignore{with ordinary belief propagation which is applicable for
pairwise models. We compare} against the TRW-S algorithm of 
\cite{kolmogorov04Convergent} which is the state of the art
algorithm for belief propagation. 
\ignore{\Majority\ potentials, on the other hand, cannot be decomposed over the edges,
and so cannot be treated using the TRW-S algorithm. So
instead} We compare the \Majority\ potential version against the stacking approach of
\cite{krishnan06:effective}.

We report results on three datasets --- the Address dataset consisting of
roughly 400 non-US postal addresses, the Cora dataset~\cite{cora} containing
500 bibliographic records, and the CoNLL'03 dataset. The training splits
were 30\%, 10\% and 100\% respectively for the three datasets,
and the parameter $\lambda$ for \Potts\ was set to 0.2,1 and 0.05.
The \Majority\ parameter $W$ was learnt generatively through 
label co-occurrence statistics in cliques seen in the training data.

Table~\ref{tab:word} reports the combined token-F1 over all labels except 'Other'.\ignore{ All the approaches post
only modest gains over the base model. However,} Unless specified, all the approaches post statistically significant gains over the base model.
For \Majority\ potentials, \mpc\ 
is superior to the stacking based approach. For the
\Potts\ version, there is no clear winner as TRW-S achieves F1 slightly
better or close to those for \mpc. But collective inference with \Majority\
potentials is more accurate than with \Potts.

Exploring \Potts\ potentials further, 
we present Figure~\ref{fig:iters}, where we plot
the accuracy of the two methods versus the
number of iterations. \mpc\ achieves its best accuracy after just
one round, whereas TRW-S takes around 20 iterations. In terms of clock time, an iteration of TRW-S cost $\sim$ 3.2s for CORA, and 
that of \mpc\ cost 3s, so \mpc\ is roughly an order of magnitude faster than TRW-S for the same accuracy levels. The comparison was 
similar for the Address dataset.

Hence the \mpc\ approach is applicable for all symmetric potentials, and
exploits their form to get higher accuracies faster than other
methods.
\begin{table}[h] 
\begin{center} 
\begin{tabular}{|l|l|c|c|c|}
\hline 
{\bf Potential} & {\bf Model} &	{\bf Address} & {\bf Cora} & {\bf CoNLL} \\
& Base & 81.5 & 88.9 & 87.0 \\ \hline
\Potts & CI & 81.9 & 89.7 & 88.8\\
 & TRW-S& 81.9 & 89.7 & - \\ \hline
\Majority & CI  & 82.2 & 89.6 & 88.8 \\
& Stacking & 81.7$^*$ & 87.5$\downarrow$ & 87.8 \\
\ignore{
& Base & 77.9 & 86.2 & 87.0 \\ \hline
\Potts & CI &78.5 & 87.3 & 88.8\\
 & TRW-S& 78.9 & 87.0 & TODO \\ \hline
\Majority & CI  & 78.6 & 87.3 & 88.8 \\
& Stacking & 78.2 & 85.9 & TODO \\}
\hline
\ignore{
Data &	StructSVM&	CI\_Potts&	CI\_Maj	&        CI\_Potts\_trws	&Stacked\_Maj \\
IITB	&77.9/69.0	&78.5/69.4	&78.6/69.7&	78.9/69.7	 &       78.1/69.5 \\
Cora&	86.2/71.4	&87.3/72.7	&87.3/73.2&	86.6/72.2&	85.9/71.3 \\
CONLL03&80.0/75.7&	80.0/75.7	&81.1/77.3	&--    &          80.2/76.2 \\
}
\end{tabular} 
\end{center} 
\caption{  
\label{tab:word}
 Token F1 scores of various approaches on collective labeling with repeated
words. Results averaged over five trials for Address and Cora. A '*' denotes statistically insignificant difference (p$>$0.05), $\downarrow$ means statistically significant loss.} 
\end{table} 
\begin{figure}
\begin{center}
\includegraphics[width=0.45\textwidth]{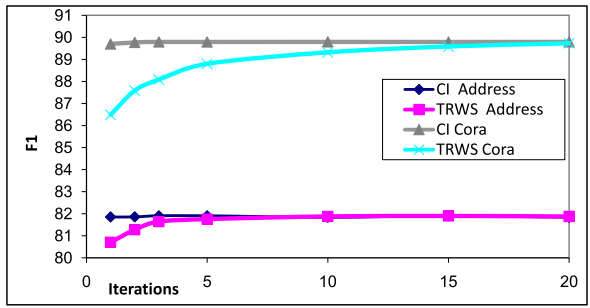}
\caption{\label{fig:iters}Accuracy with iterations for CI vs TRW-S on Cora and
Address.}
\end{center}
\end{figure}
\ignore{
Our experiments are in three parts.  

First in Section~\ref{sec-maj} we compare the accuracy and running
time of our algorithm for MAP labeling with majority potentials with
existing methods.  

Next, in Section~\ref{sec-word} we concentrate on the first
application of labeling repeated words within a document with the same
label.

Finally, in Section~\ref{sec-domain} we illustrate

\subsection{\label{sec-maj}Algorithm for Majority clique inference}


\subsection{\label{sec-word}Collective labeling of repeated Words}






\subsection{Domain adaptation}




}

\section{Conclusions and Future Work}
\label{sec:future}

We proposed a generalized collective inference framework based on
decomposable properties and symmetric potential functions to maintain
conformity in the labeling of multiple MRFs.
%
%
We perform joint MAP inference using a cluster graph that
defines special separator variables based on property values.  
The messages inside MRF clusters were modified to make them property-aware.
Special combinatorial algorithms were used at the property cliques to compute 
outgoing messages.

We demonstrated the effectiveness of the framework by applying it on a
\app\ task with a rich set of properties.  We also established
that message passing on the cluster graph is an effective solution vis
a vis cluster-oblivious approaches based on ordinary belief
propagation.

Algorithmically, we presented potential-specific combinatorial algorithms 
for inference in a clique. We gave a Lagrangian relaxation method for generating 
messages from a clique with majority potential.  This algorithm is two
orders of magnitude faster than an exact algorithm and more accurate
than other approximate approaches. We also presented the \alphapass\ algorithm
for Potts potentials, which enjoys a tight approximation guarantee of
$\frac{13}{15}$. This algorithm is sub-quadratic in the clique size. We showed
that \alphapass\ is faster and more accurate than alternatives such as TRW-S and
graph-cuts.

\subsection*{Future directions}
\ignore{
We are interested in answering various questions about collective
inference. First, we wish to understand {\em when} should we apply collective
inference. For collective inference to be useful, we need a substantial amount
of associativity present in the data, along with a limited but not highly
inaccurate baseline sequential model. When associativity is absent or is of a
different nature (e.g.~dis-associativity), collective inference might actually
degrade the performance of the base model due to error propagation. Thus, we
would like to characterize unlabeled data as suited for collective inference or
not. In addition, we would ideally like to provide lower bounds on accuracy gains
with collective inference.
}
We wish to automate the selection of important decomposable associative properties
Another issue is the domain-adaptive training of the property parameters
(e.g.~$\lambda$ for Potts). Joint training
of these parameters with the baseline model would require expensive calls to the
collective inference algorithm at each step, so a cheaper alternative has to be investigated.

Next, our property clusters are presently defined as cliques with symmetric
potentials, which have limited expressive power
So we are interested in looking at dense
weighted subgraphs instead of cliques, thus modeling that not all 
vertex-pairs have equal associativity. 

Finally, we wish to seek more applications for collective inference, and
deploy collective inference on a large scale. Although our cluster message
passing based solution is distributed and inherently  parallelizable, the
clique participants might lie on different physical machines. This, and some
other interesting scaling issues will crop up as we try to run collective
inference on a web scale.
\ignore{
As future work, we wish to focus on more complex clique inference scenarios. In
particular we would like to consider dense subgraphs instead of cliques, and
partially homogeneous clique potentials instead of the completely symmetric
potentials discussed in this paper. 

From the framework's point of view, we also wish to target the question of "when
should we apply the machinery of Collective Inference?". Collective inference is
expensive, and it would be nice to have offline data-dependent bounds on the
gains (if any) if we apply collective inference. We are also interested in
automatically bootstrapping and selecting a relevant set of properties for an
unseen domain. Finally, we also plan to study the issues in large scale
deployment of collective inference.
}

\ignore{ 
Future work includes exploring alternative types of properties, and
deploying them on other \app\ and related tasks.}


\end{document}